\documentclass[12pt]{article}

\usepackage[margin=1in]{geometry}
\usepackage{setspace}
\usepackage{amsmath}
\usepackage{amssymb}
\usepackage{amsthm}
\usepackage{xcolor}
\usepackage{subfig}
\usepackage{csvsimple}
\usepackage{graphicx}
\usepackage{rotating}
\usepackage{todonotes}
\usepackage[colorlinks=false, pdfborder={0 0 0}]{hyperref}
\usepackage{booktabs}
\usepackage{booktabs}
\usepackage{multirow}
\usepackage{authblk}

\usepackage[title]{appendix}
\usepackage[normalem]{ulem} 

\newcommand{\real}{\mathbb{R}}
\newcommand{\complex}{\mathbb{C}}

\newcommand{\norm}[2][]{\| #2 \|_{#1}}
\newcommand{\priorpower}[1]{W_{#1}}
\DeclareMathOperator*{\argmin}{arg\,min}
\newcommand{\SumOnetoN}[1][i]{\sum_{#1=1}^{n}}
\newcommand{\abs}[1]{|#1|}

\newcommand{\imageZeroEstMean}[1][i]{Y_{#1} - P_{#1} \hat{\mu}}
\newcommand{\noiseCovar}[0]{\sigma^2 I}
\newcommand{\projectedCovar}[1][i]{P_{#1} \Sigma P_{#1} ^*}
\newcommand{\matrixSpace}[3][\real]{#1^{#2 \times #3}}
\NewDocumentCommand{\SumRank}{O{i} O{r}}{\sum_{#1=1}^{#2}}
\newcommand{\inprod}[3][]{\langle #2, #3 \rangle_{#1}}
\newcommand{\projectedEigen}[1][j]{P_i v_{#1}}
\newcommand{\regWeightedEigen}[2][l]{\sqrt{r_#1} \odot v_#2}
\NewDocumentCommand{\der}{O{} m}{%
	\frac{\partial #1}{\partial #2}%
}
\newcommand{\oneover}[1]{\frac{1}{#1}}
\newcommand{\bigparants}[1]{\left( #1 \right)}
\newcommand{\approxSigma}{\hat{\Sigma}}
\newcommand{\diag}{\operatorname{diag}}
\newcommand{\projectedEigenMat}[1][i]{P_{#1}V}
\newcommand{\approxMu}{\hat{\mu}}
\newcommand{\best}[1]{\textbf{#1}} 
\newcommand{\secondbest}[1]{\underline{#1}} 
\newcommand{\projectedCovarEigenForm}[1][j]{\SumRank[#1] \projectedEigen[#1] (\projectedEigen[#1])^*}
\newcommand{\conj}[1]{\overline{#1}}
\newcommand{\backprojectedEigen}[1][j]{P_i ^* P_i v_{#1}}
\newcommand{\backprojectedImageZeroMeanEst}[1][i]{P_{#1}^* (\imageZeroEstMean[#1])}
\newcommand{\floor}[1]{\lfloor #1 \rfloor}
\newcommand{\ceil}[1]{\lceil #1 \rceil}
\newcommand{\MN}[1][]{M_N^{#1}}

\newtheorem{theorem}{Theorem}[section]
\newtheorem{lemma}[theorem]{Lemma}

\newcommand{\hide}[1]{}

\newcommand{\algname}{SOLVAR}

\title{SOLVAR: Fast covariance-based heterogeneity analysis with pose refinement for cryo-EM}
\author{
Roey Yadgar\textsuperscript{1},
Roy R. Lederman\textsuperscript{2},
and Yoel Shkolnisky\textsuperscript{3}
}

\date{\today}

\begin{document}
	
\maketitle
\footnotetext[1]{Department of Applied Mathematics, Yale University}
\footnotetext[2]{Department of Statistics and Data Science, Yale University}
\footnotetext[3]{School of Mathematical Sciences, Tel Aviv University}
\begin{abstract}
Cryo-electron microscopy (cryo-EM) has emerged as a powerful technique for resolving the three-dimensional structures of macromolecules. 
A key challenge in cryo-EM is characterizing continuous heterogeneity, where molecules adopt a continuum of conformational states. 
Covariance-based methods offer a principled approach to modeling structural variability. However, estimating the covariance matrix efficiently remains a challenging computational task.
In this paper, we present \algname~(Stochastic Optimization for Low-rank Variability Analysis), which leverages a low-rank assumption on the covariance matrix to provide a tractable estimator for its principal components, despite the apparently prohibitive large size of the covariance matrix. Under this low-rank assumption, our estimator can be formulated as an optimization problem that can be solved quickly and accurately.
Moreover, our framework enables refinement of the poses of the input particle images, a capability absent from most heterogeneity-analysis methods, and all covariance-based methods.
Numerical experiments on both synthetic and experimental datasets demonstrate that the algorithm accurately captures dominant components of variability while maintaining computational efficiency. \algname~achieves state-of-the-art performance across multiple datasets in a recent heterogeneity benchmark. The code of the algorithm is freely available at \href{https://github.com/RoeyYadgar/SOLVAR}{https://github.com/RoeyYadgar/SOLVAR}.
\end{abstract}

\section{Introduction}
Cryo-electron microscopy (cryo-EM) single-particle analysis is a method for determining the high-resolution structures of molecules, enabling biologists to analyze their function. In this method, copies of a molecule are suspended in a thin layer of vitrified ice, which preserves them in their native state, with each copy assuming a random orientation and position in the ice layer~\cite{CRYO-EMintro}. An electron microscope then captures two-dimensional noisy projection images of these molecules, resulting in up to several million particle images, which are the input to the subsequent data processing pipeline.

The classic simplified model of analyzing cryo-EM data often assumes that all the imaged particles are exact copies of the same molecular structure.
However, in practice, biological molecules exhibit structural heterogeneity that is often crucial to the biological processes in which they participate. In such cases, the imaged molecules can no longer be considered to be identical copies. 
Traditional cryo-EM analysis software models the heterogeneity as a small number of discrete states; while this approach has been very successful in many applications, it cannot resolve more complex dynamics exhibited by many molecular complexes, nor achieve high-resolution reconstructions. It is therefore necessary to develop new techniques that can accurately recover the molecule's structure while accommodating its inherent variability. This problem, known as the ''continuous heterogeneity problem,'' is an active and challenging research area.

In this paper, we introduce \algname, a new method for analyzing structural heterogeneity in cryo-EM. Our method uses the covariance based approach to heterogeneity analysis; this approach has significant theoretical advantages \cite{CovarianceMatrixEstimationKatsevich} and has been shown to outperform other methods in recent benchmarks.
Traditionally, covariance-based methods were limited in their resolution due to the sheer size of the covariance matrix they attempted to compute. More recent covariance-based methods \cite{Tagare2015-vr,3dva,recovar} tackle this limitation by exploiting low dimensional structure of the covariance matrix and estimate the matrix's principal components directly.
Despite the advancements in resolution, all covariance-based methods (and most other methods) share a critical limitation: they all require the user to provide particle images' poses, and they are inherently unable to refine these initial poses. In realistic pipelines, such poses are computed by software that ignores the heterogeneity in the data; these estimates are distorted by heterogeneity, and existing covariance-based methods are unable to correct the error.  
Building on this line of work, \algname~leverages the low rank of the covariance to reformulate the estimation into an optimization problem, which enables the direct estimation of principal components using gradient-based methods. Unlike prior methods, our framework readily accommodates particle image poses within the optimization problem.
We show, using recent benchmarks, that our algorithm, SOLVAR, outperforms existing algorithms while remaining computationally efficient.

\section{Background and Related Work}

Cryo-EM particle images are tomographic projections of the electrostatic potential of the imaged biomolecules. For convenience, we consider the particle images and volumes in Fourier space.
We will denote the Fourier transform of a particle image by $Y_i\in \complex^{N \times N}$, where $N \times N$ are the dimensions in pixels of each particle image. For brevity, we refer to the Fourier transform of a particle image also as the particle image. We will denote the Fourier transform of the electrostatic potential of the biomolecule by $X_i\in \complex^{N \times N \times N}$ and refer to it as the volume or structure.
The main advantage of the Fourier domain representation is that the complicated tomographic projection in real space can be described as a simple 2-d slice through the center of the 3-d space (the Fourier slice theorem~\cite{doi:10.1137/1.9780898719284}).

Using this notation, the imaging model under the usual weak phase assumption \cite{SCHERES2012406_relion_EM} can be written as
\begin{equation}\label{eq:ImageFormModelSingle}
	Y_i = a_i C_i T_{t_i} P_{\phi_i} X_i + e_i,
\end{equation}
where $P_{\phi_i}$ is the 2-d slice operator corresponding to orientation $\phi_i$, $T_{t_i}$ is the in-plane translation operator (in the Fourier domain) corresponding to the in-plane shift $t_i$ from the center of the particle, $C_i$ is the contrast transfer function (CTF), which is a filter related to the imaging procedure, $a_i$ is the per-image contrast scaling factor, and $e_i$ is additive Gaussian noise which we will assume throughout the paper to be white $e_i \sim N(0,\sigma^2 I)$. In practice, a preprocessing step is performed to whiten the particle stack and estimate the variance $\sigma^2$.
For the sake of simplicity, since all the transformations applied to the volume~$X_{i}$ are linear, we will merge them into a single operator $P_i = a_i C_iT_iP_{\phi_i}$ to which we will refer as the projection operator. Using this notation, the imaging model is given by
\begin{equation}\label{eq:ImageFormModel}
    Y_i = P_i X_i + e_i.
\end{equation}
We refer to the pair~$(\phi_i,t_i)$ as the pose of the $i$-th image.
The classic simplified model for analyzing cryo-EM data often assumes homogeneous data, i.e., all the imaged particles are exact copies of the same molecular structure, that is, $X_i = X$ for all~$i$. 

If the pose corresponding to each particle image is known, the problem of estimating the volume $X$ is a linear inverse problem with known matrices $P_i$. In this case, the problem can be written as a standard least squares problem with a closed form solution~\cite{SCHERES2012406_relion_EM, Punjani2017-ed, Anden_2015}. However, the poses are unknown, therefore, they must be jointly estimated with the structure~$X$. Standard homogeneous refinement algorithms employ Expectation-Maximization (EM) to iteratively update the estimated poses and structure~\cite{SCHERES2012406_relion_EM, SCHERES2012519_RELION_FSC, Punjani2017-ed}.

In a more complex heterogeneous setting, the structure $X_i$ corresponding to each particle image is different. The traditional solution is to model the heterogeneity as discrete, where there is a small number $M$ of distinct structures $\{ V_1, ... V_M\}$, and each imaged structure is a copy of one of these structures $X_i \in \{ V_1, ... V_M\}$.
This approach is implemented in some of the most popular cryo-EM data processing packages~\cite{SCHERES2012519_RELION_FSC, Punjani2017-ed}.
While this approach has been successfully used to uncover structures in many datasets with limited discrete heterogeneity, it is inadequate for analyzing continuous conformational changes.

In recent years, various approaches have been proposed for analyzing continuous heterogeneity, including, among others, manifold learning (e.g., \cite{doi:10.1073/pnas.1419276111, FRANK201661, Maji2020, Seitz2022-rb(ManifoldEM), Moscovich2020-er}), normal modes (e.g., \cite{JIN2014496, Harastani2021-iq,9616150, VUILLEMOT2022167483}), 
covariance-based methods (e.g.,\cite{CovarianceMatrixEstimationKatsevich,Anden_2015,andén2018structuralvariabilitynoisytomographic, 3dva,recovar}), non-linear methods (e.g., \cite{lederman2017continuouslyheterogeneoushyperobjectscryoem,DBLP:journals/corr/abs-1907-01589}),
and deep learning based non-linear approaches (e.g.,\cite{CryoDRGNZhong2021-nj,Levy2024.05.30.596729,Luo2023}).
These methods attempt to recover some low-dimensional description of plausible structures and their distribution.
Detailed surveys of continuous heterogeneity analysis methods can be found in~\cite{Sorzano2019-lp, TOADER2023168020}.

Many recent algorithms for analyzing continuous heterogeneity in cryo-EM use ideas and components from machine learning and neural networks. 
One notable line of algorithms is CryoDRGN \cite{CryoDRGNZhong2021-nj} and CryoDRGN2 \cite{cryodrgn2}, which are based on a Variational Autoencoder (VAE) architecture; the recent versions can estimate the poses of particle images in the ab-initio setting (where the user provides no initial pose estimates). CryoDRGN’s success sparked the development of many subsequent deep learning-based approaches. An overview of these methods can be found in \cite{Donnat2022-uy, Tang2023-ki}.

Our paper follows a linear approach to representing the continuum of structures~$X_{i}$, where different conformations are linear combinations of ``basis'' volumes, and are given explicitly by
\begin{equation}\label{eq:linear model}
	X_i = X_0 + \sum_{j=1}^{k} z_{i,j} v_j,
\end{equation}
where $X_0$ is the mean of the distribution describing the heterogeneity in the volumes,  $v_j$~are basis volumes, $z_{i,j}$ are the expansion coefficients in that basis (the latent variables), and~$k$ is the dimension of the model. 
If the volumes~$X_{i}$ were available, a set of basis vectors~$v_{j}$ could have been estimated (in principle) by estimating the covariance matrix of the volumes and using standard PCA~\cite{Hotelling1933AnalysisOA}, with the $v_j$s being the principal components of the covariance matrix. However, we recall that all we have are the particle images, $\{Y_i\}_{i=1}^n$, not the volumes, $\{X_i\}_{i=1}^n$, which we would need to compute the PCA of the volumes. Covariance-based methods in cryo-EM attempt to estimate the covariance matrix of the volumes from the input particle images and then apply PCA.
Earlier work on the covariance-based approach \cite{Penczek2011IdentifyingCS} attempted to construct multiple volumes from subsets of particle images and compute the PCA from these volumes. 
Later works (e.g., \cite{CovarianceMatrixEstimationKatsevich, Anden_2015, andén2018structuralvariabilitynoisytomographic}) revealed the remarkable fact that if the poses are known, the mean volume and the covariance matrix for computing the PCA can be estimated directly from the observed particle images $\{Y_i\}_{i=1}^{n}$ using a least squares estimation approach, which yields the explicit estimates for the mean~$\hat{\mu}$ and covariance~$\hat{\Sigma}$ of the volumes
\begin{align}
\hat{\mu} &= \argmin_{\mu} \SumOnetoN \norm{Y_i - P_i \mu}^2  + \norm[\priorpower{\mu}^{-1}]{\mu}^2, \label{eq:mean estimator}\\
\hat{\Sigma} &= \argmin_\Sigma \SumOnetoN \norm[F]{(\imageZeroEstMean)(\imageZeroEstMean)^* - \projectedCovar - \noiseCovar}^2 +
\sum_{i,j} \abs{\Sigma_{i,j}}^2 {(R_\Sigma)_{i,j}}, \label{eq:CovarLSDefinition}
\end{align}
where $\norm[\priorpower{\mu}^{-1}]{\mu}^2 = \mu^* \priorpower{\mu}^{-1} \mu$, and $\sum_{i,j} \abs{\Sigma_{i,j}}^2 {(R_\Sigma)_{i,j}}$ are regularization terms (given in detail in Appendix~\ref{sec:CovarFSCReg}).

The methods cited above estimate the entire covariance matrix of the volumes, whose size is prohibitively large for practical computations ($N^3 \times N^3$), which limited their use to very low resolutions. Subsequent covariance-based methods~\cite{Tagare2015-vr,3dva,recovar}, which have proven to be highly effective, tackle this limitation by avoiding estimating the full covariance matrix in order to approximate its eigenvectors~$v_{j}$.
The authors of \cite{Tagare2015-vr,3dva} take the PPCA (Probabilistic PCA) approach, in which the $v_{j}$s are estimated by maximizing the likelihood of the observed data. This maximization is implemented using the EM algorithm by marginalizing over the posterior distribution of the latent variables $z_{i,j}$.
The recent RECOVAR \cite{recovar} introduces new ideas to the least squares estimation approach, which together overcome many of the traditional limitations of covariance-based methods. First, it uses the particle images directly and does not estimate the full covariance matrix. Second, it is very efficient computationally and produces high-resolution volumes. Finally, it provides elaborate tools for interpreting the conformational landscape that were not available in previous software. RECOVAR has been shown to outperform many other algorithms in recent benchmarks \cite{Cryobenchjeon2025cryobenchdiversechallengingdatasets}.

Unfortunately, all existing covariance-based algorithms share a fundamental limitation: they require the user to provide the pose of each particle image as input, and cannot update these parameters. 
In practice, users must first apply a homogeneous refinement process to the input particle images, ignoring heterogeneity, to obtain estimates of the rotations and offsets, and provide these estimates to the heterogeneity analysis algorithms.
As heterogeneity becomes more complex, alignment of the particle images becomes less accurate (and, in fact, the concept of ``correct alignment'' across conformations is poorly defined), which can lead to significant errors and lower the resolution of the estimated structures.

In this paper, we introduce \algname\ -- Stochastic Optimization for Low-rank Variability Analysis -- that efficiently estimates the eigenvectors of the covariance matrix of the conformations' distribution by exploiting a special mathematical structure of the estimation problem.

The main novel contributions in \algname\ are: 
1)~Reformulating~\eqref{eq:CovarLSDefinition} so it is amenable to efficient gradient-based optimization.
2)~Introducing a maximum likelihood alternative to~\eqref{eq:CovarLSDefinition} and to the formulations in~\cite{Tagare2015-vr,3dva}.
3)~Introducing pose refinement into the covariance-based approach.

\section{\algname~algorithm}\label{sec:SOLVAR}

Despite the fact that the optimization problem in ~\eqref{eq:CovarLSDefinition} is a standard least squares problem, which is convex and has a closed form solution, the covariance matrix $\Sigma \in \matrixSpace[\complex]{N^3}{N^3}$ is prohibitively large, containing $N^6$ elements, which precludes its estimation for even moderate~$N$, both in terms of memory and run time complexity. We show below that it is possible to reformulate~\eqref{eq:CovarLSDefinition} as an optimization over the eigenvectors of~$\Sigma$ directly by imposing a low-rank assumption $r \ll N^3$ on~$\Sigma$. We also show that we can efficiently evaluate the resulting objective function and its gradients, enabling us to use stochastic gradient descent (SGD) to optimize it quickly.

To construct a low rank approximation of~$\Sigma$, we substitute $\Sigma = \sum_{i=1}^{r} v_i v_i^{*}$ into~\eqref{eq:CovarLSDefinition}, resulting in the optimization problem (see Appendix~\ref{subsec:LSDerivation} for a detailed derivation)
\begin{equation}\label{eq:CovarLSLowRank}
	\begin{aligned}
\hat{v}_1,\ldots,\hat{v}_r &= \argmin_{v_1,\ldots,v_r} f^{LS}(v_1,\ldots,v_r), \\
f^{LS}(v_1,\ldots,v_r) &= 
	\SumOnetoN \Biggl( \norm{\imageZeroEstMean}^4 -
2 (\SumRank[j] \inprod{\imageZeroEstMean}{\projectedEigen[j]}^2 + \sigma^2 \norm{\imageZeroEstMean}^2) \\
& \qquad\qquad +\SumRank[j,k] \inprod{\projectedEigen[j]}{\projectedEigen[k]}^2 + 2 \sigma^2  \SumRank[j] \norm{\projectedEigen[j]}^2 + \sigma^4 N^2 \Biggr)  \\
& \qquad\qquad +\sum_{l=1}^{m}\SumRank[j,k] \inprod{\regWeightedEigen{j}}{\regWeightedEigen{k}}^2,
 \end{aligned}
\end{equation}
with the derivative of the objective function~$f^{LS}$ given by
\begin{equation}\label{eq:CovarEigenLSDerivative}
\begin{aligned}
\der[f^{LS}]{v_k} &= 4 \Biggl[ 
\oneover{n} \SumOnetoN P_i ^* \Biggl( 
\SumRank[j] \inprod{\projectedEigen[j]}{\projectedEigen[k]} \projectedEigen[j]  \\
& \qquad \qquad \qquad - \inprod{\imageZeroEstMean}{\projectedEigen[k]}(\imageZeroEstMean)
+ \sigma^2 \projectedEigen[k] \Biggr) \\
&\qquad \qquad \qquad + \sum_{l=1}^{m} \SumRank[j] \inprod{\regWeightedEigen{j}}{\regWeightedEigen{k}} (r_l \odot v_j)
\Biggr],
\end{aligned}
\end{equation}
where the vectors $r_{l}$ are a decomposition of the regularization term $R_\Sigma$ (see~\eqref{eq:CovarLSDefinition}) such that $R_\Sigma=\sum_{l=1}^m r_l r_l ^T$, and~$\odot$ is element-wise multiplication. While the objective $f^{LS}$ in~\eqref{eq:CovarLSLowRank} is not convex and its optimum has no closed form (in contrast to~\eqref{eq:CovarLSDefinition}), evaluating~\eqref{eq:CovarLSLowRank} and~\eqref{eq:CovarEigenLSDerivative} requires $O(nr^2N^2 + mr^2N^3)$ operations, which leads to a total time complexity of $O(K(nr^2N^2 + (nm/B) r^2N^3))$ operations when using SGD with~$K$ epochs and a mini-batch size of~$B$ particle images (see Appendix~\ref{subsec:LSDerivation}). We note that the total time complexity scales linearly with the size $N^3$ of the principal components $v_j \in \complex^{N^3}$.

Moreover, optimizing~\eqref{eq:CovarLSLowRank} via SGD using~\eqref{eq:CovarEigenLSDerivative} only requires applying the operators~$P_{i}$ and their adjoints to vectors, regardless of how these operators are represented or stored. Thus, while other methods (such as~\cite{recovar,3dva}) are restricted to using the crude nearest-neighbor numerical interpolation in the implementation of the projection operator $P_i$ (see details in~\cite{recovar}),
our method can use any mathematical formulation of~$P_{i}$, trading accuracy for speed as needed. For improved accuracy, our default implementation uses trilinear interpolation with an oversampling factor of 2, a relatively common choice in cryo-EM algorithms that are not subject to the nearest-neighbor limitation.
Our implementation allows the user to choose between the nearest-neighbors interpolation, trilinear interpolation, or the NUFFT algorithm~\cite{nufft_doi:10.1137/S003614450343200X}; see~\cite{YANG2008959, Toader2025-tu} for a discussion on interpolation methods and Appendix \ref{sec:algRunTime} for runtime comparisons.

We note that the vectors $v_j$ are not required to be orthogonal for the matrix $\Sigma = \sum_{j=1}^{r} v_j v_j^T$ to be low rank and for the arguments throughout this paper. 
Therefore, we do not impose the orthogonality constraint during training and instead orthogonalize the final estimate of $\Sigma$ via SVD.

Effective, frequency-dependent, regularization has been observed to be crucial in homogeneous reconstruction methods such as~\cite{SCHERES2012519_RELION_FSC, Punjani2017-ed}, and covariance-based methods \cite{andén2018structuralvariabilitynoisytomographic}. 
A~regularization method based on splitting the data into two halves was introduced in~\cite{SCHERES2012519_RELION_FSC}, and generalized in RECOVAR~\cite{recovar} for the covariance case.
In this work, we follow RECOVAR's approach to the regularization term $R_\Sigma$ in~\eqref{eq:CovarLSDefinition}, with some modifications described in Appendix~\ref{sec:CovarFSCReg}.

\subsection{Maximum likelihood estimator}

Later in this paper, we introduce an algorithm for updating the particles' poses (section~\ref{subsec:PoseOptimization}), thereby removing a critical limitation of previous covariance-based algorithms, which rely on unrealistic pose estimates produced by homogeneous refinement algorithms.
A straightforward attempt to use the least-squares estimator~\eqref{eq:CovarLSDefinition} to jointly estimate both the principal components and the poses fails to achieve satisfactory performance (see Appendix~\ref{subsec:PoseOptDetails}). Motivated by this observation, we propose an alternative estimator that enables effective pose refinement and performs well in our experiments.

We define the maximum likelihood estimator of $\Sigma$ under the assumption that the heterogeneity is Gaussian (i.e., $X \sim N(\mu,\Sigma)$) by
\begin{multline}\label{eq:CovarMLDefinition}
\approxSigma = \argmin_{\Sigma} \SumOnetoN \Biggl((\imageZeroEstMean)^* (\projectedCovar + \noiseCovar)^{-1} (\imageZeroEstMean) 
+ \log \abs{\projectedCovar + \noiseCovar}\Biggr).
\end{multline}
While the assumption of Gaussian heterogeneity is not physically grounded and somewhat arbitrary, it is closely related to commonly used models in cryo-EM software (see, for example,~\cite{SCHERES2012519_RELION_FSC}); the resulting estimator is consistent with the particle images' formation model, and we observe that it performs well empirically. 
We add to \eqref{eq:CovarMLDefinition} a Gaussian prior $\SumRank[j] \norm{\sqrt{R_V} \odot v_j}^2$ for the eigenvectors of $\Sigma$, with $R_V = \sqrt{\diag(R_\Sigma)}$ estimated as in~\eqref{eq:CovarLSDefinition} (see Appendix~\ref{sec:CovarFSCReg}).

The optimization problem~\eqref{eq:CovarMLDefinition} does not have a closed-form solution, and it is computationally expensive since it involves inverting $N^2 \times N^2$ matrices. However, making the same low rank substitution $\Sigma = \sum_{i=1}^{r} v_i v_i^{*} = VV^*$, $V \in \matrixSpace[\complex]{N^3}{r}$, as before yields the objective function
\begin{multline}\label{eq:ML objective}
f^{ML}(V)  = \SumOnetoN \Biggl( 
\oneover{\sigma^2} \Biggl[
\norm{\imageZeroEstMean}^2 - 
\oneover{\sigma^2}(\imageZeroEstMean)^*(\projectedEigenMat)M_i^{-1}(\projectedEigenMat)^*(\imageZeroEstMean)\Biggr] \\
+\log |M_i| + N^2\log \sigma ^2 \Biggr)
+ \SumRank[j] \norm{\sqrt{R_V} \odot v_j}^2,
\end{multline}
where $M_i = I +  \oneover{\sigma^2}(\projectedEigenMat)^* \projectedEigenMat  \in \matrixSpace[\complex]{r}{r}$.
Optimizing the latter using stochastic optimization requires $O(K(n(r^2N^2 + r^3) + n/B r^2N^3))$ operations (see Appendix~\ref{subsec:MLDerivation} for a derivation of~\eqref{eq:ML objective} and the computational complexity of its optimization). Despite the $r^{3}$ term in the latter complexity, the practical runtime of the optimization is dominated by the computation of~$P_{i}$ and~$P_{i}^{*}$. So the time required to optimize~\eqref{eq:CovarLSLowRank} and~\eqref{eq:ML objective} is essentially the same (see Figure~\ref{fig:algRunTime}).

We note that~\cite{Tagare2015-vr,3dva} also solve a maximum likelihood problem. We describe the differences between these methods and SOLVAR's maximum likelihood estimation in Appendix~\ref{sec:CovObjComparison}.

\subsection{Optimization over particles' pose and contrast}\label{subsec:PoseOptimization}

As mentioned above, a typical assumption in many heterogeneity analysis algorithms, and all previous covariance-based heterogeneity algorithms, is that the poses (and usually the contrast) of the input images are known, so that the operators $P_{i} = a_i C_i T_{t_i} P_{\phi_i}$ in~\eqref{eq:ImageFormModel} are known. 
In typical pipelines for recent heterogeneity analysis algorithms, the poses $(\phi_i,t_i)$ are estimated by homogeneous refinement of the entire set of particle images before the heterogeneity analysis algorithms are invoked. However, given that the data is heterogeneous, the estimated poses are inevitably inaccurate. These inaccuracies reduce the resolution of the estimated heterogeneity and can even create spurious heterogeneity not present in the data, where an algorithm incorrectly interprets global rotation and offset as part of the structural variability.

We incorporate the particles' poses into the maximum likelihood objective~\eqref{eq:ML objective} by applying gradient descent on the particle orientations~$\phi_{i}$, and a second-order method for the particle offsets~$t_{i}$ (see Appendix~\ref{subsec:PoseOptDetails}). Additionally, we update the estimated mean volume~$\approxMu$ of the data by performing homogeneous reconstruction every $5$ epochs.
The incorporation of pose parameters renders the optimization problem susceptible to local minima, an effect that is more pronounced in the high-frequency regime, as discussed in \cite{barnett2017rapidsolutioncryoemreconstruction}.
To mitigate this effect, we apply a frequency-marching-like approach, where we apply a lowpass filter to the estimated principal components (the columns $V$ in~\eqref{eq:ML objective}) during the optimization process, and gradually increase the cutoff frequency (until arriving at the Nyquist frequency).

As pointed out in~\cite{recovar}, accurate estimation of particle images' contrasts is another important componant in the analysis of heterogeneity.
To incorporate contrast estimation into our algorithm, we apply the contrast estimation approach of~\cite{3dva} (see Appendix~\ref{subsec:PoseOptDetails}) every 5 epochs.

\section{Results}

In this section, we demonstrate \algname~on synthetic and experimental cryo-EM data.
We use \algname~to estimate the principal components of the volumes and update the pose estimates of the particle images.

As noted in \cite{sorzano_principal_2021}, reconstructing volumes with the straightforward PCA approach~\eqref{eq:linear model} requires sufficiently high SNR. RECOVAR \cite{recovar} circumvents the problem through clever use of a subset of the original particle images, which is carefully chosen and weighted based on the principal components (see \cite{recovar} for details); we adopt RECOVAR's procedure using the principal components and pose estimates produced by \algname.

When visualizing reconstructed volumes in the figures below, we use \texttt{relion\_postprocess}, which estimates the B-factor from reconstructed half-maps (according to~\cite{Rosenthal2003-le}) and sharpens the final volume. Volume figures are generated using the ChimeraX software \cite{Pettersen2021}.

\subsection{Results on synthetic datasets}

We evaluate \algname~using the synthetic datasets published in the CryoBench~\cite{Cryobenchjeon2025cryobenchdiversechallengingdatasets} benchmark, which contains 5 datasets with different types of heterogeneity. 
Since the CryoBench synthetic tests use the ground-truth poses, we do not examine the pose optimization of \algname~in this set of tests.

Comparing the outputs of different algorithms for analyzing heterogeneity in cryo-EM turns out to be a difficult problem in itself, and there is not yet a consensus in the community on the best metric.
CryoBench proposes to evaluate different methods using the Area Under the Curve (AUC) of the ``per-image FSC'' based on the Fourier Shell Correlation (FSC) metric that is very popular in cryo-EM; for details, see~\cite{Cryobenchjeon2025cryobenchdiversechallengingdatasets}.
We follow CryoBench's approach, but we note that this metric is susceptible to bias (as highlighted by~\cite{recovar}) due to the misleading nature of the FSC curve in heterogeneity analysis, where, for example, two distinct conformations could still have high correlation.

For each dataset, we use the maximum-likelihood estimator~\eqref{eq:ML objective} (with the given ground truth poses) and estimate the $r=10$ leading principal components. Table~\ref{tab:cryobenchFSCAUC} shows the AUC of the per-image FSC of our approach together with other state-of-the-art methods evaluated in CryoBench. We observe that our method achieves the highest score for most of Cryobench's datasets. We note that while CryoDRGN outperforms \algname~with the default $r=10$ rank parameter, using ~\algname~ with a higher value of $r$, which is more appropriate for a dataset with $100$ discrete classes, yields state-of-the art results; see Appendix \ref{subsec:tomotwin_appendix}.

\begin{table*}
\centering
\small
\scalebox{0.6}{
\begin{tabular}{lcccccccccc}
\toprule
\multirow{2}{*}{Method}
& \multicolumn{2}{c}{IgG-1D}
& \multicolumn{2}{c}{IgG-RL}
& \multicolumn{2}{c}{Ribosembly}
& \multicolumn{2}{c}{Tomotwin-100}
& \multicolumn{2}{c}{Spike-MD} \\
\cmidrule(lr){2-3}\cmidrule(lr){4-5}\cmidrule(lr){6-7}\cmidrule(lr){8-9}\cmidrule(lr){10-11}
& Mean (std) & Median
& Mean (std) & Median
& Mean (std) & Median
& Mean (std) & Median
& Mean (std) & Median \\
\midrule
CryoDRGN
& 0.351 (0.028) & 0.356
& 0.331 (0.016) & 0.333
& 0.412 (0.023) & 0.415
& \best{0.316 (0.046)} & \best{0.321}
& 0.340 (0.009) & 0.340
\\
CryoDRGN-AI-fixed
& 0.364 (0.002) & 0.364
& 0.348 (0.012) & 0.350
& 0.372 (0.032) &  0.375
& 0.202 (0.044) &  0.207
& 0.301 (0.012) & 0.303
\\
Opus-DSD
& 0.335 (0.026) & 0.339
& 0.343 (0.016) & 0.346
& 0.362 (0.083) &  0.382
& 0.237 (0.049) &  0.251
& 0.229 (0.027) & 0.242
\\
3DFlex
& 0.335 (0.003) & 0.335
& 0.337 (0.007) & 0.337
& - &  -
& - &  -
& 0.304 (0.011) & 0.306
\\
3DVA
& 0.349 (0.004) & 0.350
& 0.333 (0.014) & 0.335
& 0.375 (0.038) &  0.375
& 0.088 (0.04) &  0.077
& 0.324 (0.010) & 0.323
\\
RECOVAR
& \secondbest{0.386 (0.005)} & \secondbest{0.388}
& \secondbest{0.363 (0.011)} & \secondbest{0.363}
& \secondbest{0.429 (0.018)} & \secondbest{0.432}
& \secondbest{0.258 (0.109)} & \secondbest{0.254}
& \best{0.362 (0.011)} & \best{0.365}
\\
3D Class
& 0.297 (0.019) & 0.291
& 0.309 (0.01) & 0.307
& 0.289 (0.081) &  0.288
& 0.046 (0.026) &  0.037
& 0.307 (0.023) & 0.308
\\
SOLVAR (ours)
& \best{0.388 (0.003)} & \best{0.389}
& \best{0.369 (0.009)} & \best{0.368}
& \best{0.431 (0.017)} &  \best{0.432}
&  0.246 (0.104) &  0.238
&  \secondbest{0.354 (0.012)} & \secondbest{0.355}
\\
\bottomrule
\end{tabular}}
\caption{AUC of per-image FSC. Bold and underlined text denotes best and second-best results.}
\label{tab:cryobenchFSCAUC}
\end{table*}

In Figure~\ref{fig:CrybenchUMAP}, we visualize the latent embedding of the different CryoBench datasets using UMAP~\cite{mcinnes2020umapuniformmanifoldapproximation}, which are consistent with the embeddings produced by other methods evaluated by Cryobench~\cite{Cryobenchjeon2025cryobenchdiversechallengingdatasets}, and for some of the datasets, reflect more clearly the expected structure of the latent space. 
A more nuanced comparison across different algorithms is presented in the next section.

\begin{figure}
	\centering
	\includegraphics[width=1\linewidth]{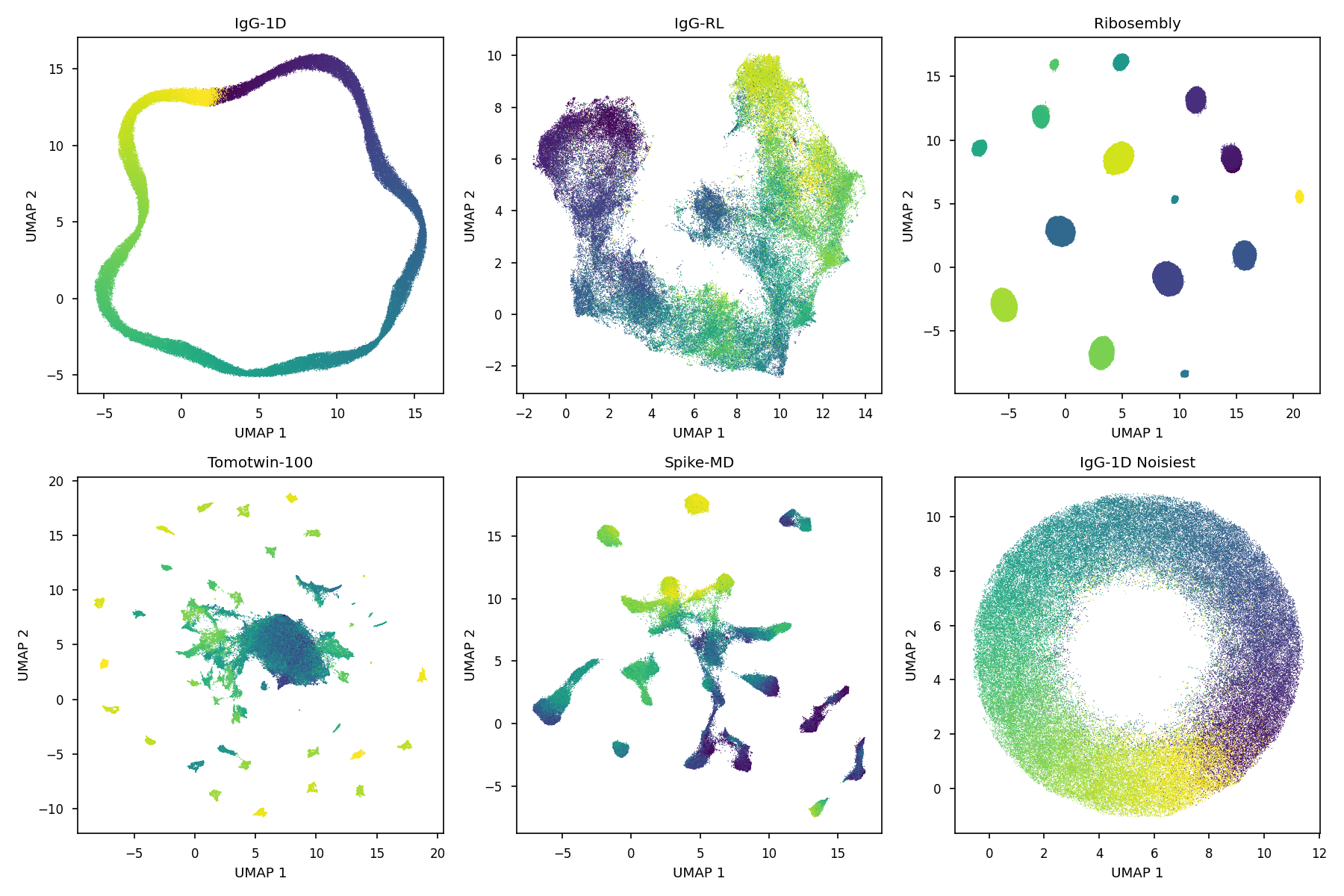}
	\caption{UMAP embedding of \algname~using 10 estimated principal components, for each dataset in Cryobench, colored by the published ground truth labels. See \cite{Cryobenchjeon2025cryobenchdiversechallengingdatasets} for a detailed description of each dataset.}
	\label{fig:CrybenchUMAP}
\end{figure}

\subsection{Particles' pose refinement and contrast correction}\label{subsec:PoseRefinement}

In order to test \algname~in a more realistic pipeline, we evaluate it in the ab-initio setting, where we use RELION's~\cite{SCHERES2012519_RELION_FSC} homogeneous refinement (\texttt{relion\_refine}) to obtain initial pose estimates for each of Cryobench's datasets, and then apply \algname~with pose optimization enabled.

In Table~\ref{tab:cryobenchPoseImprov}, we compare the error in the poses estimated by RELION with the error after the poses are refined by \algname; the error is defined as the mean over all particle images of the difference between the pose and the ground-truth pose provided in Cryobench (after a global alignment of the volume to the reference volume). As shown in Table~\ref{tab:cryobenchPoseImprov}, \algname~significantly improves RELION's estimated poses across most datasets. In addition, we compute the AUC of the per-image FSC of \algname's output (in the ab-initio setting described above). The results in Table~\ref{tab:cryobenchFSCAUCabinit} demonstrate that the RELION$\rightarrow$\algname~pipeline outperforms all other ab-initio methods evaluated by Cryobench.

To illustrate the impact of pose errors on the resulting volumes, we compare the reconstructed volumes from the IgG-1D dataset for different methods in Figure~\ref{fig:igg_comparison}. The comparison illustrates that \algname~is able to correct the poses and reconstruct a better-aligned and cleaner volume compared to other methods.

Running CryoDRGN \cite{CryoDRGNZhong2021-nj} and RECOVAR \cite{recovar} on the IgG-1D dataset (with 100,000 particle images and image size of 128) and a latent space dimension of $r=10$, takes 60 and 20 minutes respectively, while SOLVAR takes 25 minutes without optimization over poses and 90 minutes with optimization over poses. The comparison was made on a node equipped with an NVIDIA A100 GPU.

\begin{table}
	\centering
	\small
	\scalebox{0.65}{
		\begin{filecontents*}{figures/pose_opt/cryobench/pose_error_improvement.csv}
Dataset,Out of plane error (deg),In plane error (deg),Offset error (pixels)
IgG-1D,2.944 → 1.021 (65.3%),2.036 → 1.026 (49.6%),0.534 → 0.091 (82.9%)
IgG-RL,1.500 → 1.587 (-5.8%),1.356 → 1.452 (-7.1%),0.129 → 0.108 (16.8%)
Ribosembly,1.336 → 1.029 (23.0%),1.146 → 0.884 (22.8%),0.187 → 0.065 (65.0%)
Tomotwin-100,87.532 → 87.535 (-0.0%),87.697 → 87.718 (-0.0%),2.697 → 2.293 (15.0%)
Spike-MD,56.413 → 56.382 (0.1%),35.423 → 35.549 (-0.4%),7.451 → 7.846 (-5.3%)
IgG-1D Noisiest,18.399 → 18.015 (2.1%),17.576 → 17.272 (1.7%),1.111 → 0.877 (21.1%)
\end{filecontents*}
\csvautotabular[respect percent=true]{figures/pose_opt/cryobench/pose_error_improvement.csv}
	}
	\caption{Pose errors for Cryobench's datasets in the ab-initio setting. Left value is the pose error obtained by running homogeneous refinement, right value is the error after \algname~pose refinement, and the value in parentheses is the percentage of improvement made by SOLVAR. The errors for Tomotwin-100 and Spike-MD are particularly large due to the large number of discrete states (100) in Tomotwin-100 and an 'approximate symmetry' in Spike-MD. These large errors are consistent with other ab initio methods evaluated by Cryobench.
    }
	\label{tab:cryobenchPoseImprov}
\end{table}

\begin{table*}
\centering
\small
\scalebox{0.6}{
\begin{tabular}{lcccccccccc}
\toprule
\multirow{2}{*}{Method}
& \multicolumn{2}{c}{IgG-1D}
& \multicolumn{2}{c}{IgG-RL}
& \multicolumn{2}{c}{Ribosembly}
& \multicolumn{2}{c}{Tomotwin-100}
& \multicolumn{2}{c}{Spike-MD} \\
\cmidrule(lr){2-3}\cmidrule(lr){4-5}\cmidrule(lr){6-7}\cmidrule(lr){8-9}\cmidrule(lr){10-11}
& Mean (std) & Median
& Mean (std) & Median
& Mean (std) & Median
& Mean (std) & Median
& Mean (std) & Median \\
\midrule
CryoDRGN2
& 0.32 (0.062) & 0.342
& 0.301 (0.03) & 0.306
& \secondbest{0.341 (0.059)} & 0.356
& \secondbest{0.076 (0.016)} & \secondbest{0.072}
& 0.245 (0.042) & 0.260
\\
CryoDRGN-AI
& \secondbest{0.351 (0.01)} & \secondbest{0.352}
& \secondbest{0.329 (0.028)} & \secondbest{0.333}
& \secondbest{0.341 (0.083)} & \secondbest{0.367}
& 0.072 (0.015) & \secondbest{0.072}
& \secondbest{0.279 (0.017)} & \secondbest{0.281}
\\
3D Class abinit
& 0.13 (0.046) & 0.119
& 0.184 (0.022) & 0.188
& 0.144 (0.036) & 0.138
& 0.032 (0.012) & 0.031
& 0.206 (0.009) & 0.208
\\
SOLVAR-Fixed pose (ours)
& 0.348 (0.011) & 0.347
& 0.333 (0.011) & 0.330
& 0.419 (0.025) & 0.420
& 0.047 (0.010) & 0.046
& 0.311 (0.020) & 0.313
\\
SOLVAR-Pose opt (ours)
& \best{0.385 (0.006)} & \best{0.387}
& \best{0.334 (0.010)} & \best{0.332}
& \best{0.424 (0.020)} & \best{0.424}
& \best{0.103 (0.051)} & \best{0.084}
& \best{0.324 (0.028)} & \best{0.327}
\\
\bottomrule
\end{tabular}}
\caption{AUC of per-image FSC of Cryobenchs' datasets for ab-initio methods. Bold and underlined text marks methods that achieve the best and second-best metric value. For \algname~we have used RELION's homogeneous refinement as the initial poses.}
\label{tab:cryobenchFSCAUCabinit}
\end{table*}

\begin{figure}[h!]
    \centering

    \subfloat[\label{subfig:vol_comparison}]{
      \begin{minipage}[c]{0.5\textwidth}
        \centering
        \includegraphics[width=\linewidth]{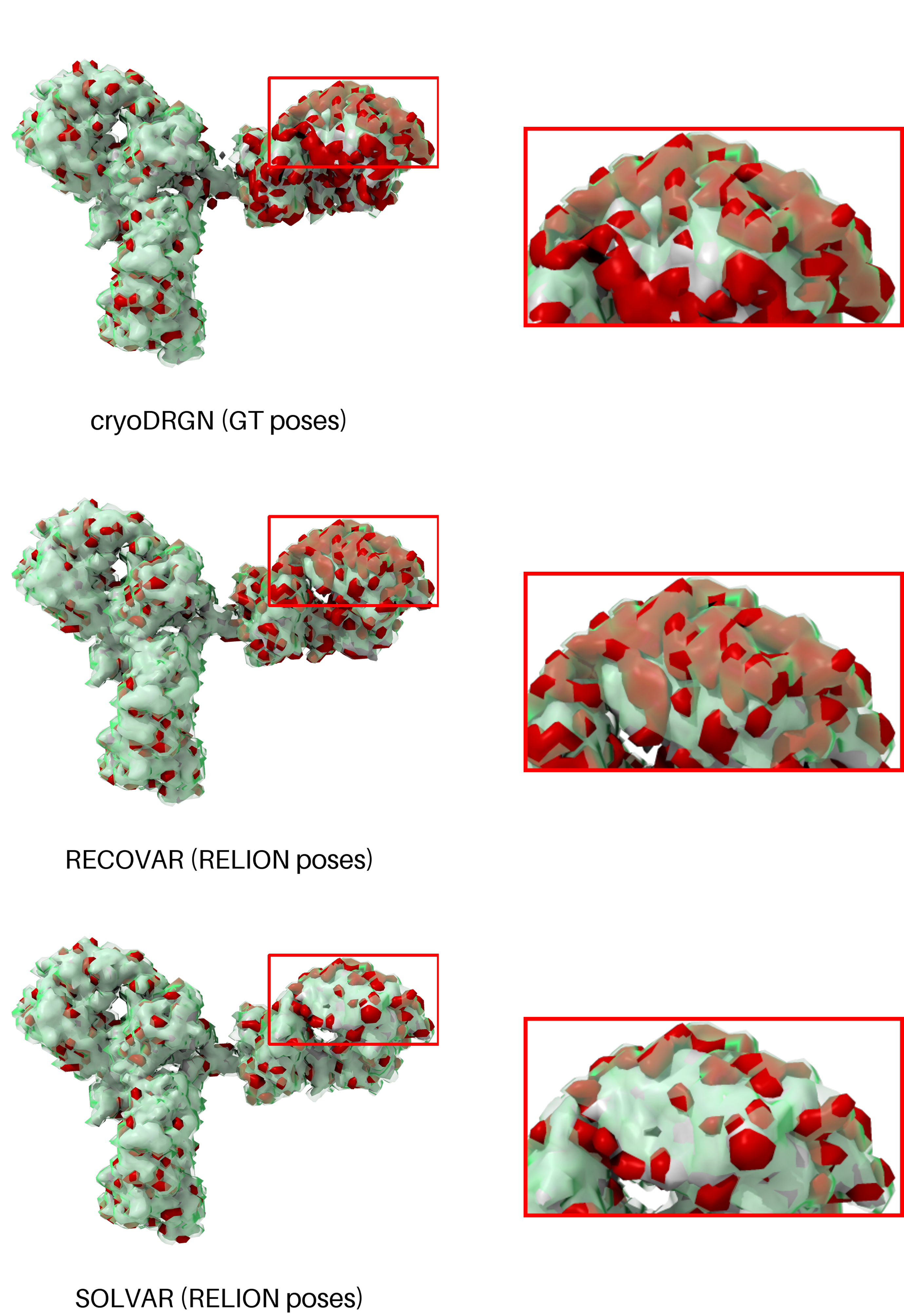}
      \end{minipage}
    }
    \hfill
    \subfloat[\label{subfig:fsc_comparison}]{
      \begin{minipage}[c]{0.45\textwidth}
        \centering
        \includegraphics[width=\linewidth]{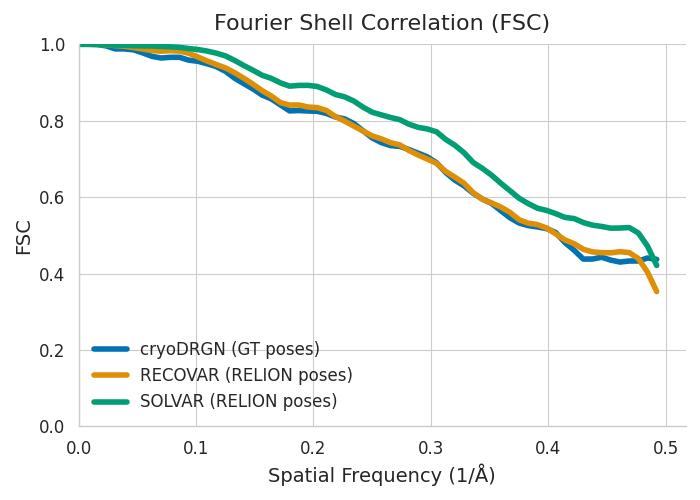}
      \end{minipage}
    }

    \caption{Comparison of reconstructed volumes using different methods for the synthetic IgG-1D dataset. 
    \protect\subref{subfig:vol_comparison}~Reconstructed volume from CryoDRGN (using ground truth poses), RECOVAR (using the poses obtained by RELION's refinement), and SOLVAR (using the poses obtained by RELION's refinement and optimizing over the particles' poses), overlayed by the ground truth conformation (green). The red parts show the binary XOR between the reconstructed volume and the overlayed ground truth volume.
    \protect\subref{subfig:fsc_comparison}~FSC curve of the three reconstructed volumes with the ground truth conformation.
    CryoDRGN produces a noisier volume despite using ground-truth poses, while RECOVAR's is not fully aligned due to errors in the input poses. SOLVAR is able to correct the poses and output an aligned and clean volume. 
    RECOVAR and SOLVAR were sharpened with \texttt{relion\_postprocess}. Since CryoDRGN does not output half-maps, it was sharpened with a B-factor of -250 \AA$^2$ (the average B-factor RELION uses for RECOVAR and SOLVAR).
    }
    \label{fig:igg_comparison}
\end{figure}

We test \algname~under a more challenging setting by introducing contrast variability to Cryobench's IgG-1D dataset. We apply a random scaling to each particle image with a factor of $\alpha \sim N(1,0.2^2)$ and use the same ab-initio approach (run homogeneous refinement and use its initial pose estimates). Figure~\ref{fig:igg_pose_refinement} shows the different pose and contrast errors during the optimization process. SOLVAR's final estimate improves the out-of-plane, in-plane, and offset errors by $56\%, 29\%$, and $69\%$, respectively, and corrects $70\%$ of the variability in the contrast.

\begin{figure}
	\centering
	\includegraphics[width=1\linewidth]{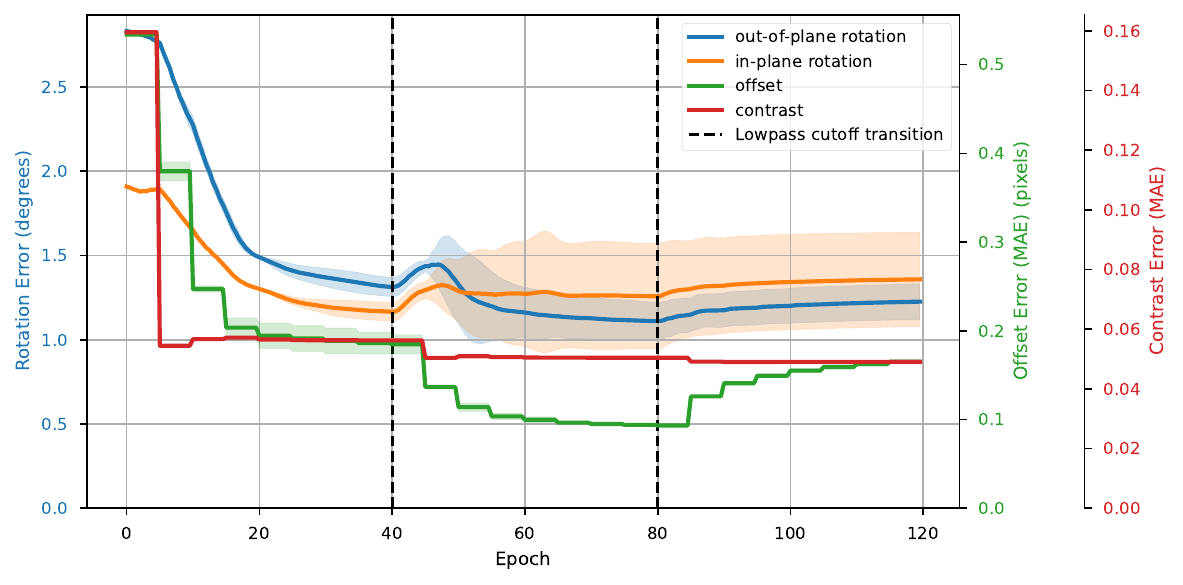}
	\caption{Particles pose error obtained by \algname~on the synthetic IgG-1D dataset. Initial pose error is obtained from homogeneous refinement process. The lowpass cutoff frequency applied to the estimated principal components increases by a factor of 2 every 40 epochs. \algname~improves the out-of-plane, in-plane, offset, and contrast errors by $56\%$, $29\%$, $69\%$, and $70\%$, respectively.}
	\label{fig:igg_pose_refinement}
\end{figure}

\subsection{Experimental datasets}

We evaluate \algname~on the EMPIAR-10076 dataset~\cite{Davis2016-qq}, which comprises 130,000 particle images of Escherichia coli large ribosomal-subunit (50S) assembly intermediates lacking the ribosomal protein L17. This dataset contains 14~distinct states identified through repeated use of 3D classification algorithms. We run \algname~on the downsampled $128\times 128$ particle images and estimate the first $r=15$ principal components. We then reconstruct the volumes from the estimated latent space coordinates using the original particle image size of $256\times 256$ pixels. 

We show the resulting UMAP embedding and reconstructed states in Figure~\ref{fig:EMPIAR10076}. We observe that the UMAP visualization of the latent space produced by \algname~is very similar to that of other methods (such as \cite{recovar, CryoDRGNZhong2021-nj}).

\begin{figure}
	\centering
	\subfloat[\label{subfig:empiar_vols}]{\includegraphics[width=0.75\textwidth]{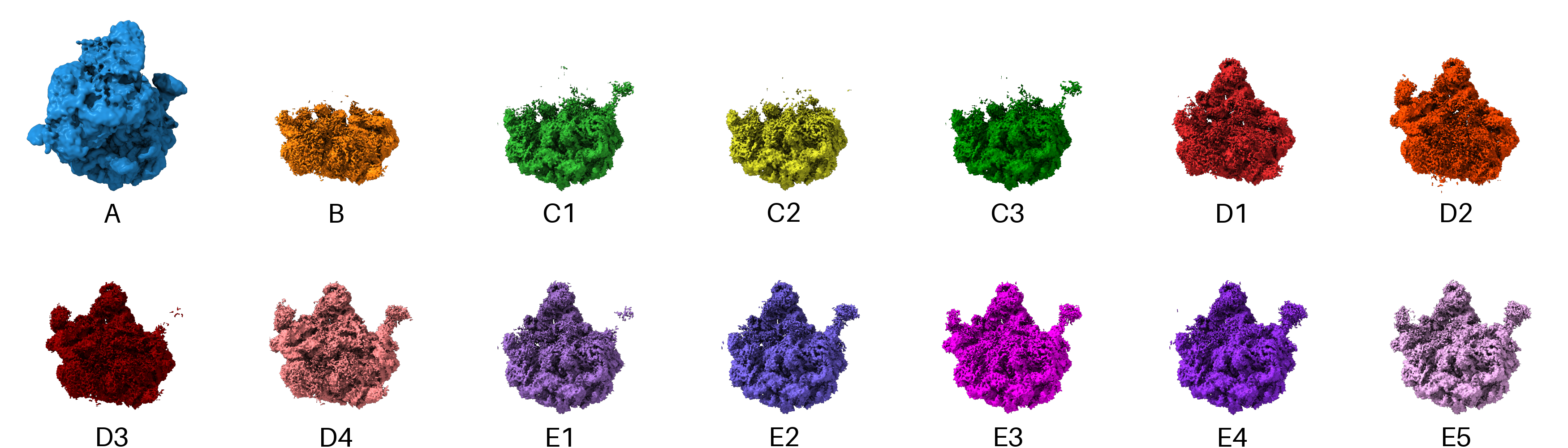}}
	\subfloat[\label{subfig:empiar_umap}]{\includegraphics[width=0.2\textwidth]{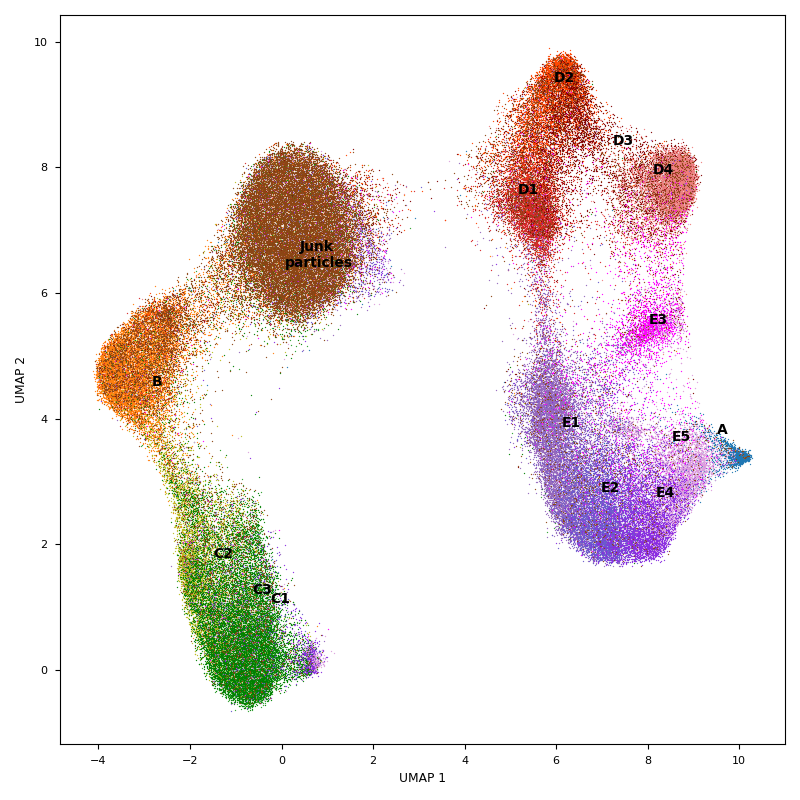}}
	\caption{\algname~results for EMPIAR-10076 using 15 estimated principal components. 
    \protect\subref{subfig:empiar_vols}~Reconstructed volumes of each identified state, colored by published labels. Volumes are sharpened with \texttt{relion\_postprocess} (except for state A) with an average B-factor of $-101$\AA$^2$.
	\protect\subref{subfig:empiar_umap}~UMAP embedding colored by the published labels.
	Each class is annotated by the mean of the latent coordinates of this class $\oneover{|S_i|}\sum_{j\in S_i} \hat{z}_j$.
	The resulting volumes and UMAP embedding are consistent with those of other methods \cite{CryoDRGNZhong2021-nj,recovar}.}
	\label{fig:EMPIAR10076}
\end{figure}

Next, we evaluate our proposed method on the EMPIAR-10180 dataset~\cite{Plaschka2017} of the pre-catalytic spliceosome (B-complex) in multiple conformational states of the SF3b/helicase region in an ab-initio setting. We ran homogeneous refinement in cryoSPARC \cite{Punjani2017-ed} to obtain initial pose estimates, then ran \algname~with pose optimization and estimated $r=10$ principal components on the entire dataset, comprising over 320,000 particle images, downsampled to $128 \times 128$ pixels. 
The results in Figure~\ref{fig:EMPIAR10180} demonstrate \algname's ability to identify continuous motion of the precatalytic spliceosome, and to identify a distinct structure with over 60,000 particle images in a visibly separate cluster in the UMAP embedding.

\begin{figure}
	\centering
	\subfloat[\label{subfig:empiar180_umap}]{\includegraphics[width=0.52\textwidth]{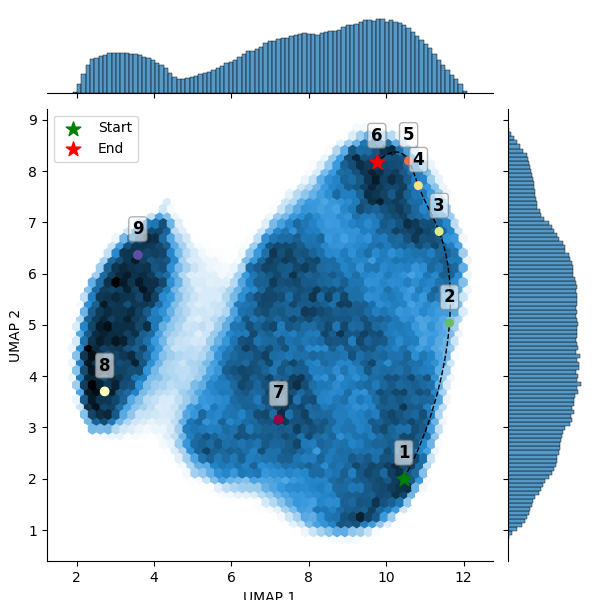}}
	\subfloat[\label{subfig:empiar180_vols}]{\includegraphics[width=0.4\textwidth]{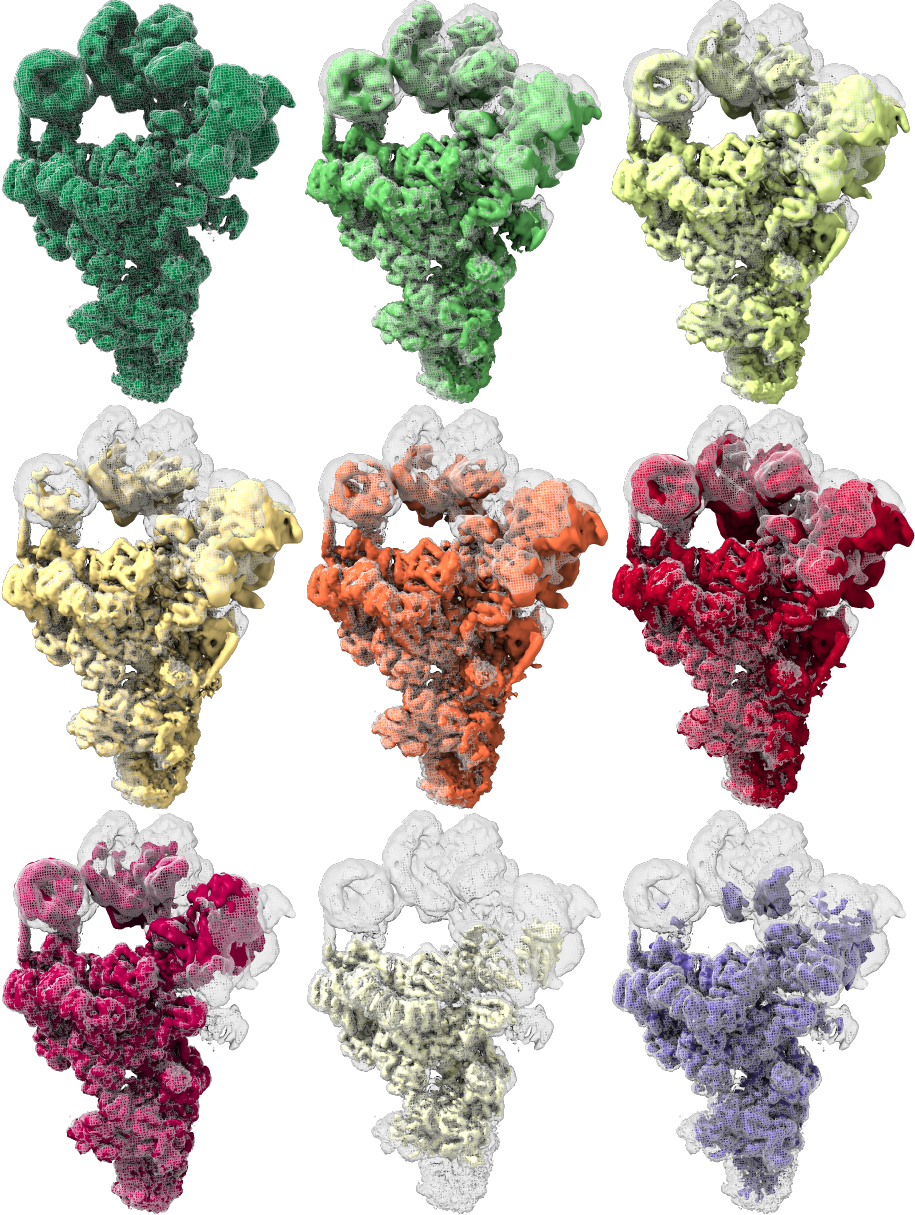}}
	\caption{\algname~results for EMPIAR-10180 using 10 estimated principal components in an ab-initio setting. 
		\protect\subref{subfig:empiar180_umap}~UMAP embedding of the resulting latent coordinates. 
		\protect\subref{subfig:empiar180_vols}~Reconstructed volumes at the trajectory points shown in~\protect\subref{subfig:empiar180_umap}. 
		The contour of the first volume (dark green) is placed on top of each volume as a reference.
        The volumes were sharpened with \texttt{relion\_postprocess} by masking the core and foot of the spliceosome (with an average B-factor of $-102$\AA$^2$) and keeping the helicase and SF3B parts unsharpened.
	}
	\label{fig:EMPIAR10180}
\end{figure}

\section{Conclusions}

A fast algorithm for analyzing continuous heterogeneity in cryo-EM data has been introduced. The method is based on direct estimation of the principal components from cryo-EM particle image data, without first estimating volumes or a costly direct estimate of the covariance matrix of volumes. Compared with previous covariance-based works, the new algorithm uses a low-rank representation of the covariance matrix that is amenable to SGD optimization and an alternative maximum-likelihood estimator. Unlike prior work, the algorithm can update pose and contrast estimates, making it applicable in realistic scenarios where  heterogeneity in the data makes it difficult to obtain optimal pose estimates separately from the analysis of the heterogeneity.

Our numerical experiments demonstrate that the algorithm achieves state-of-the-art performance on benchmarks with known poses. Furthermore, they indicate that the algorithm can be combined with traditional ab-initio software such as RELION or CryoSPARC in realistic pipelines where initial pose estimates are inaccurate and should be refined together with the analysis of heterogeneity.
The experiments demonstrate the algorithm's ability to improve pose and contrast estimates, yielding more accurate volume estimates. The algorithm is computationally efficient and scales well with the particle size $N$.

 Several directions for future work remain. First, our algorithm requires the number of components to estimate as an input from the user, which is typically not known a priori. However, the complexity of our algorithm scales as $O(r^2)$ for optimizing~\eqref{eq:CovarLSLowRank} and as $O(r^3)$ for optimizing~\eqref{eq:ML objective}. In practice, the number of estimated principal components tends to affect the runtime only linearly. Thus, one direction for improvement is to estimate the actual rank of the covariance after estimating a sufficiently large number of principal components, and prune the number of components used in the latent embedding, thus reducing noise in the latent space.

Second, the algorithm can be extended to fully ab-initio setting, or at least to settings with worse initializations, by employing the EM algorithm to jointly refine the mean and principal components as well as the particle poses, in a fashion similar to common homogeneous refinement methods. 

Finally, there remains significant potential to improve \algname’s runtime, such as implemanting efficient frequency marching ~\cite{barnett2017rapidsolutioncryoemreconstruction} within \algname's framework.

\section*{Acknowledgments}

We thank Marc Gilles for helpful discussions. We are grateful to the developers of ASPIRE, CryoDRGN, RECOVAR, and RELION for making their software open source; portions of our implementation rely on functionalities provided by these packages. We also thank the Yale Center for Research Computing (YCRC) for providing computing resources and support. The work was supported by NIH/NIGMS (1R35GM157226), the Alfred P. Sloan Foundation (FG-2023-20853), and the Simons Foundation (1288155).

\begin{appendices}

\section{Methods}
\subsection{General details}
\algname~ is implemented in PyTorch \cite{PytorchDBLP:journals/corr/abs-1912-01703}, and uses its autodiff framework to compute the derivatives of optimization objectives \eqref{eq:CovarLSLowRank},\eqref{eq:ML objective} (i.e. derivative expressions are not explicitly implemented). We use the Adam optimizer \cite{kingma2017adammethodstochasticoptimization} with a learning rate of $10^{-6}$ and a batch size of $1024$. 

\subsection{Covariance FSC-based Regularization}\label{sec:CovarFSCReg}
We follow the approach of RECOVAR \cite{recovar}, which generalizes the FSC-based regularization scheme to the covariance case (see \cite[Supplementary section C]{recovar}) with some modifications.
In contrast to RECOVAR, which applies this approach on each column of the covariance matrix independently, we use it on each ``frequency shell pair'' or ``sub-matrix'' of the covariance in Fourier space. Let~$S_i$ denote the frequencies which correspond to the $i$-th shell, that is, $S_i = \{f \in \MN[3] | \norm{f} \in [i-0.5,i+0.5] \}$ where $\MN = \{ -\floor{N/2} , \ldots , \ceil{N/2-1} \}$. Then, the sub-matrix $\Sigma_{S_i,S_j}$ contains the rows and columns of~$\Sigma$ corresponding to the $i$-th and $j$-th shells, respectively. 
We define the covariance FSC between two covariance estimates $A,B$ to be the correlation of each pair of frequency shells
\newcommand{\shellPair}[0]{S_{i},S_{j}}
\begin{equation}\label{eq:CovarFSCProperty}
	CovarFSC(A, B)_{i,j} = \frac{\inprod[\shellPair]{A}{B}}{\norm[\shellPair]{A} \norm[\shellPair]{B}},
\end{equation}
where,
\begin{equation*}
	\inprod[\shellPair]{A}{B} = \sum_{i\in S_{i} , j\in S_{j}} A[i][j] \conj{B}[i][j].
\end{equation*}

We note that we can compute the inner products in~\eqref{eq:CovarFSCProperty} efficiently using Lemma~\ref{prop:FroLowRankInnerProd}, since any sub-matrix of a low-rank matrix is also low-rank, which results in 
\begin{equation*}
	\inprod[\shellPair]{A}{B} = \SumRank[m_1][r] \SumRank[m_2][r] \inprod[S_{i}]{a_{m_1}}{b_{m_2}} \conj{\inprod[S_{j}]{a_{m_1}}{b_{m_2}}}.
\end{equation*}
In fact, we can also compute these inner products for all pairs of shells at once ($\inprod[\shellPair]{A}{B} \; , \forall i,j \in \{1,\ldots,N\}$) by denoting $G \in \matrixSpace[\complex]{N^3}{r^2}$ and $C= GG^*$ where $G$ is defined by
\begin{equation*}
	G_{k,l} = \inprod[S_{k}]{a_{m_1}}{b_{m_2}} \; , l = r (m_1-1) + m_2.
\end{equation*}
Then, $C_{i,j} = (GG^*)_{i,j} = \inprod[\shellPair]{A}{B}$. The multiplication $C = GG^{*}$ requires $O(N^3r^2)$ operations, while computing $G$ itself requires computing the inner products between all pairs of eigenvectors and over all $N$ shells, resulting in a complexity of $O(\sum_i r^2 |S_{i}|)= O(r^2 N^3)$ operations (that is since the different shells cover the volume), which is the same complexity we had for computing the regular Frobenius inner product in Lemma~\ref{prop:FroLowRankInnerProd}. 

The matrix $R_\Sigma$ in~\eqref{eq:CovarLSDefinition} is then computed by
\begin{equation*}
	(R_\Sigma)_{k,l} = \frac{1 - CovarFSC(A, B)_{i,j}}{CovarFSC(A, B)_{i,j}} T_{i,j} \; \forall k \in S_i, l \in S_j,
\end{equation*}
where $T_{i,j} = \oneover{\abs{S_i}\abs{S_j}}\SumOnetoN[m] \sum_{k' \in S_i , l' \in S_j} \abs{(p_m)_{k'} (p_m)_{l'}}^2$ and $(p_m)_{k'}$ is the $k'$-th diagonal entry of the projection operator of the $m$-th particle (in the nearest-neighbor interpolation of the projection operator in Fourier domain). We then approximate the matrix $R_\Sigma$ with its first $m=5$ eigenvectors $R_\Sigma = \sum_i^m r_i r_i^T$.

\subsection{Optimizing particles' pose}\label{subsec:PoseOptDetails}

We use different techniques to refine the rotation, offset, and predict the contrast of each particle image. For the rotation, we use gradient descent over the particle rotation vector $\theta_i \in \real^3$ where its magnitude represents the angle of rotation in radians $\norm{\theta_i} \leq \pi$ and $\frac{\theta_i}{\norm{\theta_i}}$ is its axis of rotation. We note that the additional computation of the gradient with respect to the rotation $\der[f]{\theta_i}$ does not increase the overall computational complexity of the algorithm.

We note that the nearest-neighbor interpolation is not differentiable with respect to the rotation vector; therefore, this form of pose refinement cannot be applied when this interpolation scheme is used.

For the optimizing over the offsets of the particle images, we use a combination of Newton's method and a line search algorithm, i.e. $\delta^{(n)} = \delta^{(n-1)} - \alpha_n H^{-1} \der[f]{\delta}|_{\delta=\delta^{(n-1)}}$ where $\delta \in [-N,N]^2$ is the particle image offset, and $\alpha \in [0,1]$ is the taken to be the largest $\alpha$ such that $f(\delta^{(n)}) < f(\delta ^ {(n-1)})$. We have found that replacing standard gradient descent with Newton's method and line search is crucial for being able to improve the initial pose estimation consistently (See Figure~\ref{fig:igg_pose_refinement_comparison}).
We update the particle image offsets every $5$ epochs using $10$ Newton iterations.

For the estimation of the contrast per particle image, we follow the approach proposed by 3DVA~\cite{3dva}, and estimate the contrast of the $i$-th particle image by 
\begin{equation*}
	\hat{\alpha}_i = \frac{\bigparants{P_i (\approxMu + \hat{V}\hat{z}_i)}^* Y_i}{\norm{P_i (\approxMu + \hat{V}\hat{z}_i)}^2},
\end{equation*}
where $\hat{z}_i$ are the latent coordinates given by
\begin{equation*}
	\hat{z}_i = ((P_i\hat{V})^* (P_i \hat{V}) + \sigma^2I)^{-1} (\oneover{\sigma^2} P_i \hat{V})^* Y_i.
\end{equation*}
We then incorporate the estimated contrast by updating the projection operator $\bar{P_i} = \hat{\alpha_i} P_i$ during the optimization process.

In order to improve the convergence of the optimization process and mitigate its susceptibility to local-minima, we apply a lowpass filter on the estimated principal components after each stochastic iteration $V^{(n+1)} = L_f (V^{(n)} - \alpha \der[f]{V})$ where $L_f$ applies the lowpass filter with cutoff frequency $f$ on each of the columns of $V$. We start by $f = \pi/4$ and multiply it by a factor of $2$ every $K=40$ epochs, until reaching $f=\pi$.

Figure~\ref{fig:igg_pose_umap} shows the effect of pose and contrast correction on the resulting UMAP embedding, and that the maximum likelihood estimator performs considerably better compared to the least squares estimator when optimizing over the particle poses. We defer to future work the theoretical analysis of the difference between the two objectives in the context of pose estimation.

\begin{figure}
	\centering
	\begin{minipage}{0.4\linewidth}
		\centering
		\textbf{LS - no pose refinement} \\[0.5ex]
		\includegraphics[width=\linewidth]{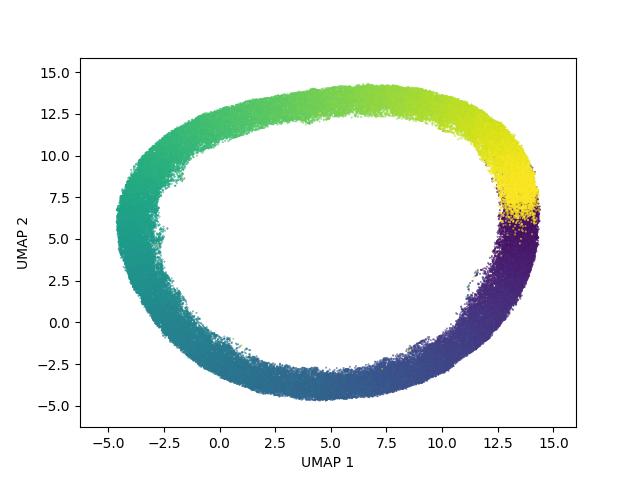}
	\end{minipage}
	\begin{minipage}{0.4\linewidth}
		\centering
		\textbf{ML - no pose refinement} \\[0.5ex]
		\includegraphics[width=\linewidth]{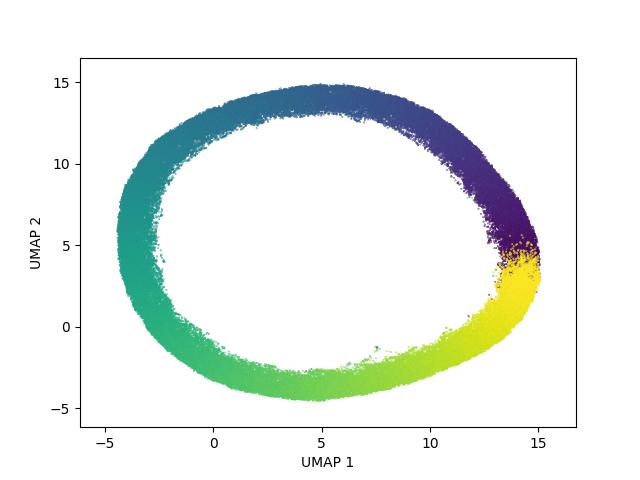}
	\end{minipage}
	\vspace{1em}
	\begin{minipage}{0.4\linewidth}
		\centering
		\textbf{LS with pose refinement} \\[0.5ex]
		\includegraphics[width=\linewidth]{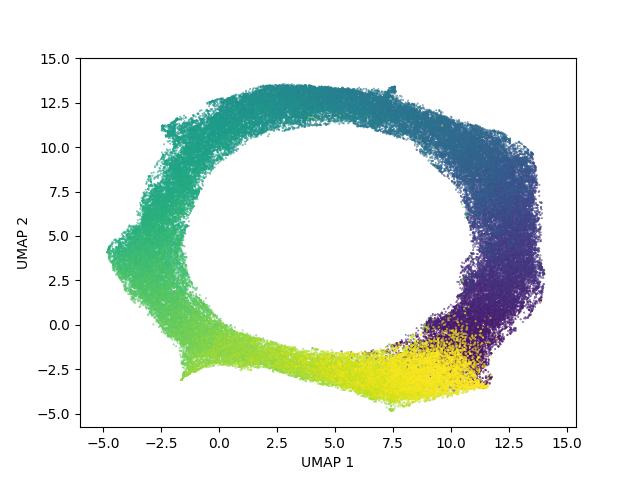}
	\end{minipage}
	\begin{minipage}{0.4\linewidth}
		\centering
		\textbf{ML with pose refinement} \\[0.5ex]
		\includegraphics[width=\linewidth]{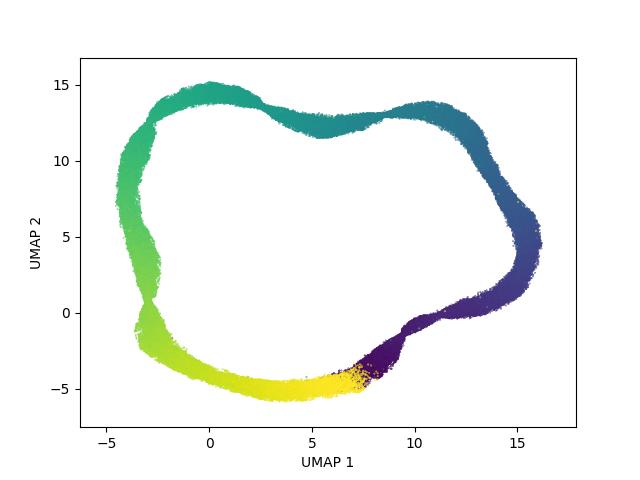}
	\end{minipage}
	
	\caption{Comparison of UMAP embeddings on the synthetic IgG-1D Cryobench dataset, with and without (bottom and top) optimization over poses and contrast, using least squares (left) and maximum likelihood (right) estimators. The least-squares estimator does not perform well when optimizing over the particle poses. The maximum likelihood estimator, combined with pose optimization, is able to reduce the width of the manifold and uncovers the single degree of freedom in this dataset.}
	\label{fig:igg_pose_umap}
\end{figure}

\begin{figure}
	\centering
	\includegraphics[width=1\linewidth]{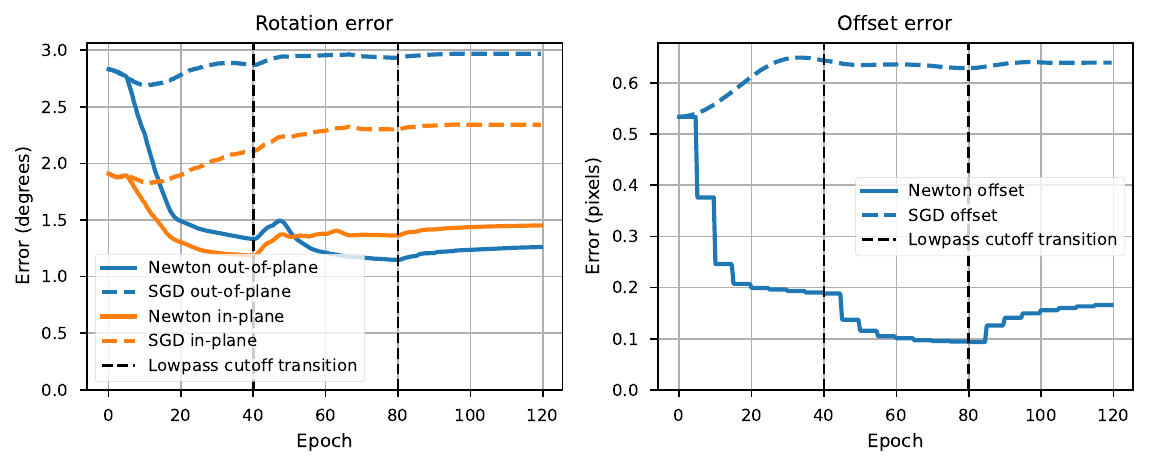}
	\caption{SGD vs Newton's method for offset refinement for the IgG-1D dataset (poses initialized by RELION homogeneous refinement). Standard SGD on the particle offset degrades the error of the initial homogeneous refinement, which leads to a degradation in the particle's rotation as well.}
	\label{fig:igg_pose_refinement_comparison}
\end{figure}

\section{Derivation of objective function}
\subsection{Least squares estimator}\label{subsec:LSDerivation}

We start with the following useful identity.

\begin{lemma}\label{prop:FroLowRankInnerProd}
	Let $A,B \in \matrixSpace[\complex]{n}{n}$ be two low-rank positive semi-definite (PSD) matrices, where $A = \SumRank[i][r_1] a_i a_i^*$ and $B = \SumRank[i][r_2] b_i b_i^* $. Then, their Frobenius inner product $\inprod[F]{A}{B}$ (that is, their dot product when considered as vectors) satisfies $\inprod[F]{A}{B} = \SumRank[i][r_1]\SumRank[j][r_2] \inprod{a_i}{b_j}^2$.
\end{lemma}
\begin{proof}
	\begin{equation*}
		\begin{aligned}
			\inprod[F]{A}{B} &= \operatorname{tr}(A^*B) = \operatorname{tr} \bigparants{(\SumRank[i][r_1] a_i a_i^*)(\SumRank[i][r_2] b_i b_i^*)} \\ &= \operatorname{tr} \bigparants{\SumRank[i][r_1] \SumRank[j][r_2] a_i a_i^* b_j b_j^*} \\
			&= \operatorname{tr} \bigparants{\SumRank[i][r_1] \SumRank[j][r_2] \inprod{a_i}{b_j} a_i b_j^*} \\
			&= \SumRank[i][r_1] \SumRank[j][r_2] \inprod{a_i}{b_j} \operatorname{tr} \bigparants{ a_i b_j^*} 
			= \SumRank[i][r_1] \SumRank[j][r_2] \abs{\inprod{a_i}{b_j}}^2,
		\end{aligned}
	\end{equation*}
	where the last equality follows since the diagonal of a rank-1 matrix is simply the element-wise product of the two vectors, that is $(a_i b_j^*)_{kk} = (a_i)_k \conj{(b_j)}_k$.
\end{proof}

Lemma~\ref{prop:FroLowRankInnerProd} allows us to compute the Frobenius inner product of two low-rank PSD matrices using just inner products of their eigenvectors. Thus, given that we have the eigen-decomposition of the matrices, we can compute their inner product in $O(r_1 r_2 n)$ operations. Since the optimization objective in~\eqref{eq:CovarLSDefinition} contains a Frobenius norm of a low-rank matrix (plus a scaled identity matrix), we can easily rewrite it using the eigenvectors of the covariance matrix, as follows.
\begin{lemma}\label{lemma:EigenFrobeniusIdentity}
	It holds that
	\begin{multline}\label{eq:EigenFrobeniusIdentity}
		\norm[F]{(\imageZeroEstMean)(\imageZeroEstMean)^* - \projectedCovar - \noiseCovar}^2 = \norm{\imageZeroEstMean}^4 \\
		-2 \bigparants{\SumRank[j] \inprod{\imageZeroEstMean}{\projectedEigen[j]}^2 + \sigma^2 \norm{\imageZeroEstMean}^2} \\
		+ \SumRank[j,k] \inprod{\projectedEigen[j]}{\projectedEigen[k]}^2 + 2 \sigma^2  \SumRank[j] \norm{\projectedEigen[j]}^2 + \sigma^4 N^2.  
	\end{multline}
\end{lemma}
\begin{proof} 
	Expanding the left-hand-side of~\eqref{eq:EigenFrobeniusIdentity}, we get    \begin{multline}\label{eq:EigenFrobeniusDerivation1}
		\norm[F]{(\imageZeroEstMean)(\imageZeroEstMean)^* - (\projectedCovar + \noiseCovar)}^2 
		= \norm[F]{(\imageZeroEstMean)(\imageZeroEstMean)^*}^2 \\
		-2 \inprod[F]{(\imageZeroEstMean)(\imageZeroEstMean)^*}{\projectedCovar + \noiseCovar} + 
		\norm[F]{\projectedCovar + \noiseCovar}^2.   
	\end{multline}    
	The first term on the right-hand-side of~\eqref{eq:EigenFrobeniusDerivation1} is a rank-1 matrix and so
	\begin{equation}\label{eq:EigenFrobeniusDerivation2}
		\norm[F]{(\imageZeroEstMean)(\imageZeroEstMean)^*}^2 = \norm{\imageZeroEstMean}^4.
	\end{equation}
	The second term on the right-hand side of~\eqref{eq:EigenFrobeniusDerivation1} can be split into an inner product of a rank-1 matrix with a rank-$r$ matrix, plus an inner product between a rank-1 matrix and a scaled identity matrix, that is        
	\begin{equation}\label{eq:EigenFrobeniusDerivation3}
		\begin{aligned}
			\inprod[F]{(\imageZeroEstMean)(\imageZeroEstMean)^*}{\projectedCovar + \noiseCovar} 
			&= \inprod[F]{(\imageZeroEstMean)(\imageZeroEstMean)^*}{\projectedCovarEigenForm} \\
			&\qquad\qquad\qquad+\inprod[F]{(\imageZeroEstMean)(\imageZeroEstMean)^*}{\noiseCovar} \\
			&= \SumRank[j] \inprod{\imageZeroEstMean}{\projectedEigen}^2 + \sigma^2 \norm{\imageZeroEstMean}^2,
		\end{aligned}
	\end{equation}
	where we have used Lemma~\ref{prop:FroLowRankInnerProd} and the identity $\inprod[F]{yy^*}{I} = \operatorname{tr}(yy^*) = \norm{y}^2$.
	
	The Frobenius norm in the third term on the right-hand-side of~\eqref{eq:EigenFrobeniusDerivation1} can be split similarly
	\begin{equation}\label{eq:EigenFrobeniusDerivation4}
		\begin{aligned}
			\norm[F]{\projectedCovar + \noiseCovar}^2 &= \norm[F]{\projectedCovarEigenForm}^2 - 2\inprod[F]{\projectedCovarEigenForm}{\sigma^2I} + \norm[F]{\sigma^2 I}^2 \\
			&= \SumRank[j,k] \inprod{\projectedEigen[j]}{\projectedEigen[k]}^2 +2 \sigma^2 \SumRank[j] \norm{\projectedEigen[j]}^2 + \sigma^4 N^2.
		\end{aligned}
	\end{equation}
	Combining \eqref{eq:EigenFrobeniusDerivation2}--\eqref{eq:EigenFrobeniusDerivation4} yields the desired result.
\end{proof}

	\begin{lemma}\label{lemma:RegTermIdentity}
		If $R_{\Sigma} = \sum_{l=1}^{m} r_l r_l^T$, then
		\begin{equation*}
			\sum_{i,j} \abs{\Sigma_{i,j}}^2 {(R_\Sigma)_{i,j}} = \sum_{l=1}^{m}\SumRank[j,k] \inprod{\regWeightedEigen{j}}{\regWeightedEigen{k}}^2.
		\end{equation*}		
	\end{lemma}
	\begin{proof}
		Let $C_l = \sum_{j=1}^{r} (\regWeightedEigen{j}) (\regWeightedEigen{j})^*$. Then, $C_l$ is of rank $r$, and Lemma~\ref{prop:FroLowRankInnerProd} gives 
		\begin{equation*}
			\norm[F]{C_l}^2 =\SumRank[j,k] \inprod{\regWeightedEigen{j}}{\regWeightedEigen{k}}^2.
		\end{equation*}
		On the other hand, $C_l$ satisfies
		\begin{equation*}
			(C_l)_{i,j} = \SumRank[k] \sqrt{(r_l)_i (r_l)_j} (v_k)_i (v_k)^*_j,
		\end{equation*}
		and so $\abs{(C_l)_{i,j}}^2 = \abs{\Sigma_{i,j}}^2 \sqrt{(r_l)_i (r_l)_j}$. Summing over all $l = 1,\ldots,m$ yields the desired result.
	\end{proof}

Lemmas \ref{lemma:EigenFrobeniusIdentity} and~\ref{lemma:RegTermIdentity} enable us to rewrite the least squares objective in Equation~\eqref{eq:CovarLSDefinition} as a function of the covariance eigenvectors in Equation~\eqref{eq:CovarLSLowRank}.
Evaluating Equation~\eqref{eq:CovarLSLowRank} at some point requires performing $n(r+1)$ projections, namely, computing $P_{i}\approxMu$ and $P_{i}v_{j}$, for $i=1,\ldots,n$ and $j=1,\ldots,r$, which requires $O(nrN^2)$ operations, as well as computing $O(nr^2)$ inner products between the projected volumes $P_iv_j$, also in a complexity of $O(nrN^2)$ operations. Evaluating the regularization term (last row of~\eqref{eq:CovarLSLowRank}) requires $mr^2$ inner products between principal components of size $N^{3}$, resulting in a complexity of $O(mr^2N^3)$ operations. In total, the runtime complexity required to evaluate $f^{LS}$ in~\eqref{eq:CovarLSLowRank} is $O(nrN^2 + mr^2N^3)$.

Next, we derive the gradient of $f^{LS}$ and show that evaluating this gradient can be done in the same complexity as evaluating $f^{LS}$.
\begin{lemma}\label{lemma:EigenLSDerivative}
	The gradient of $f^{LS}(v_1,\ldots,v_r)$ satisfies
	\begin{multline}\label{eq:CovarEigenLSDerivative2}
		\der[f^{LS}]{v_k} = 4 \Biggl[ 
		\oneover{n} \SumOnetoN P_i ^* \left ( 
		\SumRank[j] \inprod{\projectedEigen[j]}{\projectedEigen[k]} \projectedEigen[j] - 
		\inprod{\imageZeroEstMean}{\projectedEigen[k]}(\imageZeroEstMean) 
		+ \sigma^2 \projectedEigen[k] \right ) \\
		+ \sum_{l=1}^{m} \SumRank[j] \inprod{\regWeightedEigen{j}}{\regWeightedEigen{k}} (r_l \odot v_j)
		\Biggr].
	\end{multline}
\end{lemma}
\begin{proof}
	\newcommand{\derEigen}{\der{v_k}}
	We take the derivative of each term in $f^{LS}$ that depends on the vectors $v_1,\ldots,v_r$ to obtain
	\begin{equation*}
		\begin{aligned}
			\derEigen \SumRank[j] \inprod{\imageZeroEstMean}{\projectedEigen[j]}^2 
			&= 2\inprod{\imageZeroEstMean}{\projectedEigen[k]} \derEigen \inprod{\imageZeroEstMean}{\projectedEigen[k]} \notag\\
			&= 2\inprod{\imageZeroEstMean}{\projectedEigen[k]} \backprojectedImageZeroMeanEst 
		\end{aligned}
	\end{equation*}
	and
	\begin{equation}\label{eq:projInnerProdDerivative}
		\begin{aligned}
			\derEigen \SumRank[j,l] \inprod{\projectedEigen[j]}{\projectedEigen[l]}^2 
			&= \derEigen \bigparants{2\sum_{j \neq k}^{r} \inprod{\projectedEigen[j]}{\projectedEigen[k]}^2  + \inprod{\projectedEigen[k]}{\projectedEigen[k]}^2} \\
			&= 4 \sum_{j \neq k}^{r} \inprod{\projectedEigen[j]}{\projectedEigen[k]}\backprojectedEigen[j] + 4 \inprod{\projectedEigen[k]}{\projectedEigen[k]}\backprojectedEigen[k]\\
			&= 4 \SumRank[j] \inprod{\projectedEigen[j]}{\projectedEigen[k]}\backprojectedEigen[j].
		\end{aligned}
	\end{equation}
	The first equality in~\eqref{eq:projInnerProdDerivative} follows since inner products where $j,l \neq k$ do not contribute to the derivative, other pairs where either $j$ or $l$ are equal to $k$ are summed twice, while the last inner product where $j,l = k$ is summed once. Next,
	\begin{equation*}
		\begin{aligned}
			\derEigen \bigparants{2\sigma^2 \SumRank[j] \norm{\projectedEigen}^2} &= 4\sigma^2 \backprojectedEigen[k], \\
			\derEigen \bigparants{\sum_{l=1}^{m}\SumRank[j,k] \inprod{\regWeightedEigen{j}}{\regWeightedEigen{k}}^2} &= 4 \sum_{l=1}^{m} \SumRank[j] \inprod{\regWeightedEigen{j}}{\regWeightedEigen{k}} (r_l \odot v_j),
		\end{aligned}
	\end{equation*}
	where the derivation of the last expression is the same as in~(\ref{eq:projInnerProdDerivative}).
	
	Combining all expressions and extracting the back-projection operator $P_i^*$ from the first 3~terms results in the expression in~\eqref{eq:CovarEigenLSDerivative2}.
\end{proof}

We note that evaluating the gradient in Equation~\eqref{eq:CovarEigenLSDerivative2} for all $k=1,\ldots,r$ and for a single particle image (fixed~$i$), requires computing $P_i v_k$ for $k=1,\ldots,r$, as well as calling $P_{i}^{*}$ for each~$k=1,\ldots,r$, amounting to $O(rN^2)$ operations. In addition, the derivative with respect to a single~$v_{k}$ requires a linear combination of the projected principal components $\SumRank[j] \inprod{\projectedEigen[j]}{\projectedEigen[k]}\projectedEigen[j]$. This term contains inner products of all pairs of projected principal components, which can be written as $ \inprod{\projectedEigen[j]}{\projectedEigen[k]} = S_{j,k} =  ((P_iV)(P_iV)^*)_{j,k}$, hence computed in $O(r^2N^2)$ operations (as a matrix multiplication of matrices of sizes $r \times N^2$ by $N^2 \times r$). Finally, computing $\SumRank[j] \inprod{\projectedEigen[j]}{\projectedEigen[k]}\projectedEigen[j]$ for each $v_k$ can also be written as $(P_iV)S$, which is a matrix multiplication of sizes $N^2 \times r$ by $r \times r$, hence is also done in $O(r^2N^2)$ operations. The term $\sum_{l=1}^{m} \SumRank[j] \inprod{\regWeightedEigen{j}}{\regWeightedEigen{k}} (r_l \odot v_j)$ is computed in the exact same manner, but here, the inner products and linear combinations are of volumes (instead of images) and so require $O(mr^2N^3)$ operations. However, this term does not need to be computed for each particle image. Thus, summing the gradient over all observed particle images results in a total complexity of $O(nrN^2 + mr^2N^3)$ operations, which is same as the complexity required to evaluate $f^{LS}$ itself.

\subsection{Maximum likelihood estimator}\label{subsec:MLDerivation}

We derive the runtime complexity required to evaluate~\eqref{eq:CovarMLDefinition}, whose objective function is
\begin{equation}\label{eq:fMLi}
	f_{i}^{ML}(V) = (\imageZeroEstMean)^T (\projectedCovar + \noiseCovar)^{-1} (\imageZeroEstMean) + \log \abs{\projectedCovar + \noiseCovar}
\end{equation}
and compute its gradient.

\begin{lemma}\label{lemma:MLEigenQRProperty}
	Let $V \in \matrixSpace[\complex]{N^3}{r}$ be the eigenvectors of $\Sigma = VV^*$. Then, Equation~\eqref{eq:fMLi} (the log-likelihood with respect to a single particle image) satisfies 
	\begin{multline}\label{eq:EigenML}
		f_{i}^{ML}(V)  =
		\oneover{\sigma^2} \bigparants{
			\norm{\imageZeroEstMean}^2 - \oneover{\sigma^2}(\imageZeroEstMean)^*(\projectedEigenMat)M_i^{-1}(\projectedEigenMat)^*(\imageZeroEstMean)
		} \\ \qquad \qquad \qquad \qquad \qquad \qquad \qquad \qquad \qquad
		+\log |M_i| + N^2\log \sigma ^2,
	\end{multline}
	where $M_i = I + \oneover{\sigma^2}(\projectedEigenMat)^* \projectedEigenMat  \in \matrixSpace[\complex]{r}{r}$.
\end{lemma}
\begin{proof}
	Since $\Sigma = VV^*$, we have that
	\begin{equation}
		\projectedCovar + \noiseCovar =  \projectedEigenMat V^*P_i^* + \noiseCovar = (\projectedEigenMat)(\projectedEigenMat)^* + \noiseCovar.
	\end{equation}
	Using the Woodbury matrix identity, which states that 
	$$(A + UCS)^{-1} = A^{-1} - A^{-1}U(C^{-1} + SA^{-1} U)^{-1} S A^{-1},$$ 
	and setting $A = \noiseCovar$, $U=\projectedEigenMat$, $C = I$, $S=(\projectedEigenMat)^*$ yields
	\begin{align}
		((\projectedEigenMat) (\projectedEigenMat)^* + \noiseCovar)^{-1}
		= \oneover{\sigma^2}I - \oneover{\sigma^4} (\projectedEigenMat) M_i^{-1} (\projectedEigenMat)^*,
		\label{eq:WoodburyQRIdentity}
	\end{align}
	which results in
	\begin{multline}\label{eq:term1}
		(\imageZeroEstMean)^*(\projectedCovar + \noiseCovar)^{-1} (\imageZeroEstMean) 
		\\ = \oneover{\sigma^2} \bigparants{
			\norm{\imageZeroEstMean}^2 - \oneover{\sigma^2}(\imageZeroEstMean)^* (\projectedEigenMat) M_i^{-1} (\projectedEigenMat)^* (\imageZeroEstMean)
		}.
	\end{multline}
	As for the determinant term $\log \abs{\projectedCovar + \noiseCovar}$, we can use the matrix determinant lemma, which states that $\abs{A + US^*} = \abs{I + S^* A^{-1} U} \abs{A}$, with $A = \noiseCovar$, $U = \projectedEigenMat$, $S = \projectedEigenMat$, and get
	\begin{equation}\label{eq:term2}
		\log \abs{\projectedCovar + \noiseCovar} = \log \abs{I + \oneover{\sigma^2}(\projectedEigenMat)^* \projectedEigenMat} + \log \abs{\sigma^2I_{N^2\times N^2}} = \abs{M_i} + {N^2}\log \sigma^2.
	\end{equation}
	Combining~\eqref{eq:term1} and~\eqref{eq:term2} concludes the proof.
\end{proof}

Using Lemma \ref{lemma:MLEigenQRProperty}, all required matrix products can be computed in no more than $\max(N^2r,r^2)$ operations. To see that, we first note that the term $(\projectedEigenMat)^* (\imageZeroEstMean)$ is a matrix-vector product where $(\projectedEigenMat)^*$ is of size $r \times N^2$ and $\imageZeroEstMean$ is a vector of size $N^2$, hence computing $(\projectedEigenMat)^* (\imageZeroEstMean)$ requires $O(N^2 r)$ operations. Next, once this vector of size~$r$ is computed, we need to compute the bilinear form $(\imageZeroEstMean)^* (\projectedEigenMat) M_i^{-1} (\projectedEigenMat)^* (\imageZeroEstMean)$, with the same vector $(\projectedEigenMat)^* (\imageZeroEstMean)$ on both sides. Since $M_i$ is of size $r \times r$, this computation is done in $O(r^2)$ operations. Furthermore, inverting $M_i = I + \oneover{\sigma^2} (\projectedEigenMat)^* \projectedEigenMat$ and computing its determinant  can be done in $O(r^3)$ operations. Additionally, we need to compute the projected mean $P_i \approxMu$ and covariance eigenvectors $\projectedEigenMat$, which is done in $O(rN^2)$ operations. In total, evaluating~\eqref{eq:fMLi} for all particle images can be done in $O(n(N^2r^2 + r^3))$ operations. 

We next analyze the complexity of evaluating the gradient $\der[f_{i}^{ML}]{V}$, without deriving its explicit form.
\begin{lemma}[\cite{Petersen2008}]\label{prop:MatrixCalc}
	Let $A \in \matrixSpace{n}{m}$, $x\in \real^n,y\in \real^m$. Then, 
	\begin{align}
		\der[x^T A y]{A} &= \der[y^T A^T x]{A} = xy^T,\label{eq:MatrixCalc1}\\
		\der[x^T AA^T y]{A} &= (xy^T  + y x^T)A. \label{eq:MatrixCalc2}
	\end{align}
	Furthermore, for $n=m$
	\begin{align}
		\der[x^T A^{-1} y]{A} &= -A^{-T} xy^T A^{-T}, \label{eq:MatrixCalc3} \\        
		\der[\log |A|]{A} &= A^{-T}. \label{eq:MatrixCalc4}
	\end{align}
\end{lemma}
\newcommand{\fML}{f_{i}^{ML}}
\newcommand{\gML}{g_{i}^{ML}}
\begin{lemma}
	Let $\fML(V)$ be the function defined in~\eqref{eq:EigenML},
	then, the gradient $\der{V} \SumOnetoN \fML(V)$ can be evaluated in $O(n(rN^3 log N + r^2N^2 + r^3))$ operations using NUFFT (in spatial domain), or in $O(n(r^2N^2 + r^3))$ operations using nearest-neighbor/tri-linear interpolation (in Fourier domain).
\end{lemma}
\begin{proof}
	We start by noting that $\fML$ in~\eqref{eq:EigenML} also depends on the matrix $M_i \in \matrixSpace[\complex]{r}{r}$, which is in itself a function of $\projectedEigenMat$. Denote $Z_i = \projectedEigenMat$ and define $\gML(M_i,Z_i) = \fML(V)$, which can be written as
	\begin{equation*}
		\gML(M_i,Z_i) = \oneover{\sigma^2} \bigparants{
			\norm{\imageZeroEstMean}^2 - \oneover{\sigma^2}(\imageZeroEstMean)^* Z_i M_i^{-1}Z_i^*(\imageZeroEstMean)
		} +\log |M_i| + N^2\log \sigma ^2.
	\end{equation*}
	Then, using the chain rule, we have that
	\begin{equation}\label{eq:MLBackprop}
		\der[\fML]{V} = P_i^* \der[\fML]{Z_i} = P_i^* \bigparants{\sum_{k,l=1}^{r} (\der[\gML]{M_i})_{k,l} \der[(M_i)_{k,l}]{Z_i} + \der[\gML]{Z_i}}.
	\end{equation}
	First, we derive the gradient of $\gML$ with respect to $M_i$ (defined in Lemma~\ref{lemma:MLEigenQRProperty}).
	Using Lemma~\ref{prop:MatrixCalc} (specifically~\eqref{eq:MatrixCalc3} and~\eqref{eq:MatrixCalc4}) with respect to $M_i$ results in
	\begin{equation}\label{eq:MLderivativeM}
		\der[\gML]{M_i} = \oneover{\sigma^4} M_i^{-1}(Z_i^* (\imageZeroEstMean)(\imageZeroEstMean)^* Z_i) M_i^{-T} + M_i^{-1},
	\end{equation}
	where $M_i^{-T} = M_i^{-1}$ since $M_i$ is symmetric.
	Note that just like in the case of evaluating $f_{i}^{ML}$ itself, this gradient only requires computing $Z_i^*(\imageZeroEstMean)$ and $(M_i)^{-1}$ once, resulting in a complexity of $O(N^2 r)$ and $O(r^3)$ operations, respectively. Furthermore, the term $M_i^{-1} Z_i^* (\imageZeroEstMean)$ is a multiplication of a matrix of size $r \times r$ and a vector of size $r$ and so can be evaluated in $O(r^2)$ operations. Finally, we need to evaluate the outer product of the latter resulting vector with itself, requiring an additional $O(r^2)$ operations. Overall, the gradient with respect to $M_i$ is evaluated in $O(N^2r + r^3)$ operations.
	
	Next, we set $Z_i= Z_0 + \alpha dZ$, and take the directional derivative in the direction $dZ$, that is, we evaluate $\left. \der[\fML]{\alpha} \right|_{\alpha=0}$ (we will omit the notation $\left. \cdot \right|_{\alpha=0}$ in the derivation below, except for where it is necessary for clarity). Using the chain rule, we have that    
	\begin{align}
		\der[\fML]{\alpha} &= \sum_{k,l = 1}^{r} \bigparants{(\der[\gML]{M_i})_{k,l} (\der[M_i]{\alpha})_{k,l} + (\der[\gML]{Z_i})_{k,l} \der[(Z_i)_{k,l}]{\alpha}} \nonumber \\
		&= \operatorname{tr}(\der[\gML]{M_i}^* \der[M_i]{\alpha}) +\operatorname{tr}(\der[\gML]{Z_i}^* \der[Z_i]{\alpha}), \label{eq:MLbackpropMAlpha}
	\end{align}
	where the latter follows since $\sum_{k,l=1}^r A_{k,l} B_{k,l} = \operatorname{tr}(A^*B)$.
	The gradient of $M_i$ with respect to $\alpha$ satisfies
	\begin{equation*}
		\left. \der[M_i]{\alpha} \right|_{\alpha=0} = 
		\oneover{\sigma^2} \left. \der{\alpha} \bigparants{(Z_0 + \alpha dZ)^*(Z_0 + \alpha dZ)} \right|_{\alpha=0} = 
		\oneover{\sigma^2} \bigparants{(dZ)^*Z_i + Z_i^* (dZ)},
	\end{equation*}
	and so substituting $\der[M_i]{\alpha}$ in~\eqref{eq:MLbackpropMAlpha} gives    \begin{equation}\label{eq:MLbackpropMDirectional}
		\begin{aligned}
			\der[\fML]{\alpha} &=
			\operatorname{tr}(\der[\gML]{M_i}^* \oneover{\sigma^2}\bigparants{(dZ)^* Z_i + Z_i^* (dZ)}) + \operatorname{tr}(\der[\gML]{Z_i}^* \der[Z_i]{\alpha})\\
			&= \oneover{\sigma^2}\bigparants{\operatorname{tr}(\der[\gML]{M_i}^* (dZ)^* Z_i) + \operatorname{tr}(\der[\gML]{M_i}^* Z_i^*(dZ))} + \operatorname{tr}(\der[\gML]{Z_i}^* \der[Z_i]{\alpha})\\
			&=        \oneover{\sigma^2}\bigparants{\operatorname{tr}((dZ)^* Z_i\der[\gML]{M_i}^*) + \operatorname{tr}(\der[\gML]{M_i}^* Z_i^*(dZ))} + \operatorname{tr}(\der[\gML]{Z_i}^* \der[Z_i]{\alpha}).
		\end{aligned}
	\end{equation}
	Next, we set $dZ = E_{kl}$, where $E_{kl}$ has value~1 at index $k,l$ and 0 everywhere else (i.e. the derivative w.r.t each entry of $Z_i$). Using the identity $\operatorname{tr}(A^* E_{kl}) = \operatorname{tr}(E_{kl}^* A) = A_{k,l}$, it follows that the matrix $\der[\fML]{Z_i}$, which is defined element-wise by
	\begin{multline*}
		\bigparants{\der[\fML]{Z_i}}_{k,l} = \left. \der[\fML]{\alpha} \right|_{\alpha=0,dZ = E_{kl}}= \oneover{\sigma^2}\bigparants{\operatorname{tr}((E_{kl})^* Z_i\der[\gML]{M_i}^*) + \operatorname{tr}(\der[\gML]{M_i}^* Z_i^*(E_{kl}))} \\ + \operatorname{tr}(\der[\gML]{Z_i}^* E_{kl}),
	\end{multline*} results in
	\begin{equation}\label{eq:MLbackpropZ}
		\der[\fML]{Z_i} = \oneover{\sigma^2}\bigparants{Z_i\der[\gML]{M_i}^*  + Z_i\der[\gML]{M_i}^*} +\der[\gML]{Z_i} = \frac{2}{\sigma^2} Z_i \der[\gML]{M_i} +\der[\gML]{Z_i},
	\end{equation}
	where we have used the symmetry of 
	$\der[\gML]{M_i}$.
	
	Moving on to the term $\der[\gML]{Z_i}$, it can be expanded by
	using (\ref{eq:MatrixCalc1}) twice -- once with $x^* = (\imageZeroEstMean)^*$, $A = Z_i$, $y = M_i^{-1} Z_i ^*(\imageZeroEstMean)$, and once with $x = \imageZeroEstMean$, $A^* = Z_i ^*$, $y^*= (\imageZeroEstMean)^* Z_i M_i^{-1}$ -- which results in
	\begin{multline}\label{eq:MLderZ}
		\der[\gML]{Z_i}  = \oneover{\sigma^4}\bigparants{(\imageZeroEstMean)((\imageZeroEstMean)^* Z_i) M_i^{-T} + (\imageZeroEstMean)((\imageZeroEstMean)^* Z_i) M_i^{-1}} \\
		= \frac{2}{\sigma^4} (\imageZeroEstMean)((\imageZeroEstMean)^* Z_i) M_i^{-1}.
	\end{multline}
	Analyzing the computational complexity required to evaluate~\eqref{eq:MLbackpropZ}, we see that the matrix multiplication $Z_i \der[\gML]{M_i}$ is evaluated in $O(N^2 r^2)$ operations due to the sizes of $Z_i \in \matrixSpace{N^2}{r}$ and $\der[\gML]{M_i} \in \matrixSpace{r}{r}$. The latter combined with the complexity of evaluating $\der[\gML]{M_i}$ (explained above~\eqref{eq:MLderivativeM}), which is  $O(N^2 r + r^3)$ operations, amounts to $O(N^2r^2+r^3)$ operations. Next, evaluating $\der[\gML]{Z_i}$ requires evaluating the term $((\imageZeroEstMean)^* Z_i) M_i^{-1}$, which is exactly the same quantity as in~\eqref{eq:MLderivativeM} (and is evaluated in $O(N^2r + r^3)$ operations). Moreover~\eqref{eq:MLderZ} requires an outer product of $\imageZeroEstMean$ with the term  $((\imageZeroEstMean)^* Z_i) M_i^{-1}$, which takes an additional $O(N^2r^2)$ operations.
	Overall, the complexity of evaluating the derivative $\der[\fML]{Z_i}$ does not exceed a complexity of $O(N^2r^2 + r^3)$ operations.
	
	Finally, the gradient with respect to $V$ is $\der[f_{i}^{ML}]{V} = P_i^* \der[f_{i}^{ML}]{Z_i}$, which simply means performing~$r$ back-projections (notice that this was also the case for the derivative of the least squares estimator in Lemma~\ref{lemma:EigenLSDerivative}), hence requires $O(rN^3 \log N)$ or $O(rN^2)$ operations, depending if NUFFT or nearest-neighbor/trilinear interpolation in Fourier domain is used. Combining the back-projection complexity with the complexity of the gradient with respect to $Z_i$ shown above completes the proof.
\end{proof}

\section{Comparison of objective function with other covariance-based methods}\label{sec:CovObjComparison}
3DVA \cite{3dva} maximizes the likelihood $\SumOnetoN \log p(y_i ,z_i| V)$ given the principal components and the latent variables (the conditioning is also taken over the mean~$\approxMu$, but it is assumed to be known and fixed) under the model $y_i \sim N(P_i (\approxMu + Vz_i), \sigma^2I)$, which is equivalent to minimizing the objective
\begin{equation}
    f^{3DVA}(V,z)=\SumOnetoN \norm{Y_i - P_i (\approxMu + Vz_i)}^2.
\end{equation}
Since the latent variables $z$ are unknown, the EM algorithm is employed to marginalize over the posterior of $z$. However, in practice, 3DVA omits marginalization over $z$ and instead uses a point estimate of~$z$ in the maximization step over~$V$.

RECOVAR \cite{recovar} solves the least-squares problem~\eqref{eq:CovarLSDefinition} for the covariance matrix. Instead of solving the full least-squares problem, only a subset of the columns of the covariance matrix is computed, which is possible due to the diagonal structure of the least-squares operator under the nearest-neighbor approximation of the projection operator $P_i$.
RECOVAR then takes advantage of the low-rank assumption of the covariance matrix, which enables to estimate the eigenvectors from the subset of columns.

SOLVAR has two options for the objective function to be optimized. In the first one, the least squares problem~\eqref{eq:CovarLSDefinition} is rewritten as a function of the principal components $V$ (see Appendix~\ref{subsec:LSDerivation}). Compared to RECOVAR~\cite{recovar}, this approach obviates the heuristic selection of subsets of columns for the computation.
In the second option, SOLVAR maximizes the likelihood given the covariance matrix, given by $\SumOnetoN \log p(y_i| \Sigma)$, under the model $y_i \sim N(P_i \approxMu, \projectedCovar + \noiseCovar)$, which is equivalent to minimizing the objective \eqref{eq:CovarMLDefinition}. The latter objective can also be rewritten in a form amenable to efficient optimization, which results in~\eqref{eq:ML objective} (see Appendix~\ref{subsec:MLDerivation}). Compared to 3DVA~\cite{3dva}, this approach eliminates the need to optimize over the latent variables $z_{ij}$ of~\eqref{eq:linear model}.

\section{Cryobench results}

\subsection{UMAP comparison with existing methods}
We compare the UMAP embedding produced by~\algname~with that of other methods evaluated by Cryobench. Figure~\ref{fig:igg_umap_comparison} shows the UMAP embedding on IgG-1D dataset, and observe that out of the continuous heterogeneity methods, only RECOVAR \cite{recovar} and \algname~are able to fully capture the one-dimensional nature of this dataset. For a comparison of on the remaining datasets in CryoBench, we refer the reader to the CryoBench paper \cite{Cryobenchjeon2025cryobenchdiversechallengingdatasets}.

\begin{figure}
    \centering
    \subfloat[\label{subfig:solvar_umap}]{\includegraphics[width=0.46\textwidth]{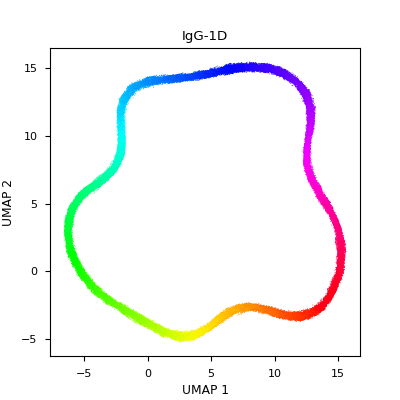}}
    \subfloat[\label{subfig:cryobench_umap_comp}]{\includegraphics[width=0.5\linewidth]{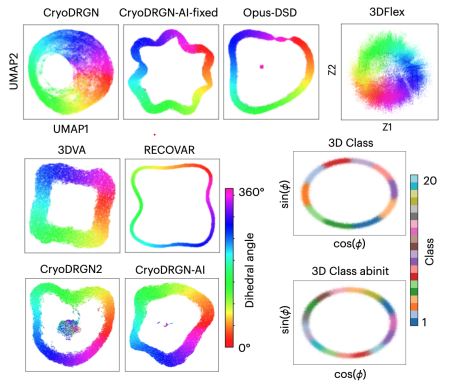}}.
    \caption{Comparison of UMAP embedding of~\algname~\protect\subref{subfig:solvar_umap} and other methods \protect\subref{subfig:cryobench_umap_comp} evaluated by Cryobench on the IgG-1D dataset. Subfigure~\protect\subref{subfig:cryobench_umap_comp} is adapted from
    Jeon et al.~\cite{Cryobenchjeon2025cryobenchdiversechallengingdatasets},
    licensed under CC BY 4.0.}
    \label{fig:igg_umap_comparison}
\end{figure}

\subsection{Tomotwin-100}\label{subsec:tomotwin_appendix}
For the Tomotwin-100 dataset, our method fails to achieve competitive results in Table~\ref{tab:cryobenchFSCAUC}. We note that this dataset is particularly challenging for covariance-based methods (indeed, RECOVAR and 3DVA also fail to achieve good results) since it contains 100 discrete volumes with highly varying structures \cite{Cryobenchjeon2025cryobenchdiversechallengingdatasets}, which requires a high-dimensional linear subspace to capture their variability. 
It is important to note  that single particle cryo-EM datasets like these are somewhat rare, and it is unlikely that such datasets will be accompanied by very good pose estimates if such pose estimates are produced earlier in the pipeline using software that does not already classify the particles when it estimates their poses. Nevertheless, our algorithm can estimate a larger number of principal components in this case, which considerably improves the per-image FSC results from 0.246 using 10 principal components to 0.312 and 0.333 using 25 and 50 principal components, respectively (see Figure~\ref{fig:tomotwin_rank_comparison}), allowing it to achieve the highest per-image FSC metric out of all the methods evaluated by CryoBench.

\begin{figure}
	\centering
	\subfloat[\label{subfig:tomotwin_umap}]{\includegraphics[width=1\textwidth]{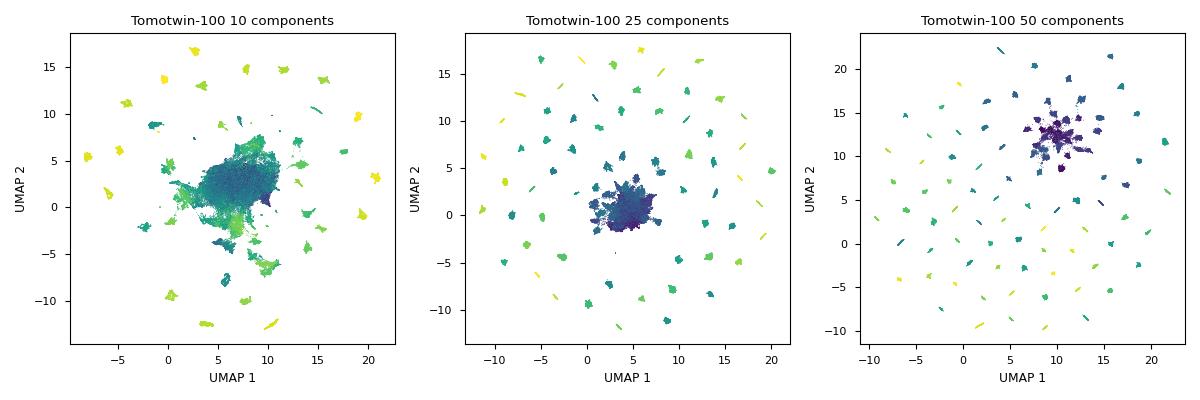}}\\
	\subfloat[\label{subfig:tomotwin_vols1}]{\includegraphics[width=1\textwidth]{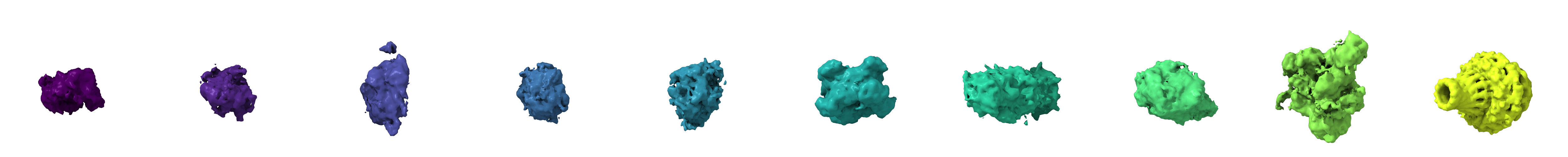}}\\
	\subfloat[]{\includegraphics[width=1\textwidth]{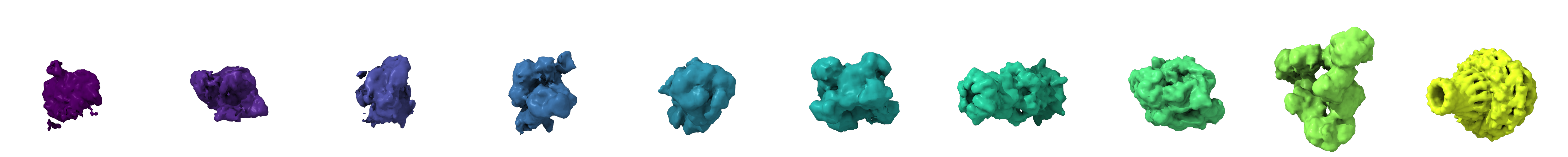}}\\
	\subfloat[\label{subfig:tomotwin_vols3}]{\includegraphics[width=1\textwidth]{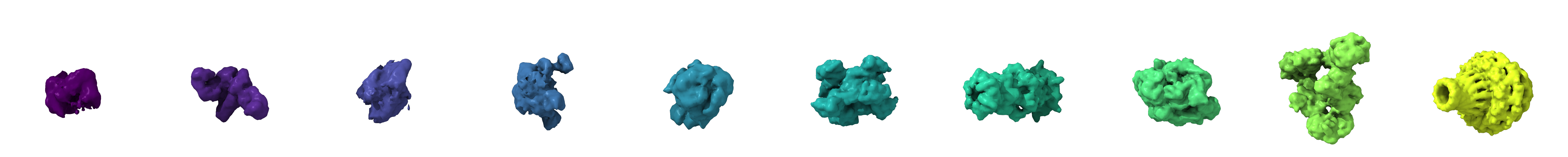}}\\
	\subfloat[\label{subfig:tomotwin_vols4}]{\includegraphics[width=1\textwidth]{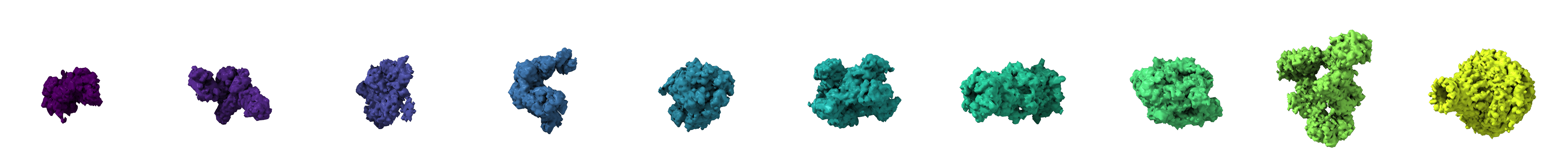}}\\
	\caption{\protect\subref{subfig:tomotwin_umap}~UMAP embedding of Cryobench's Tomotwin-100 dataset using 10, 25, and 50 estimated principal components. A larger number of principal components allows to distinguish between more conformations.
    \protect\subref{subfig:tomotwin_vols1}-\protect\subref{subfig:tomotwin_vols3} Reconstructed volumes of the Tomotwin-100 dataset, using 10,25 and 50 estimated principal components.
    \protect\subref{subfig:tomotwin_vols4} Ground truth states of the Tomotwin-100 dataset.
    }
	\label{fig:tomotwin_rank_comparison}
\end{figure}

\section{Run time analysis}\label{sec:algRunTime}
As mentioned in Section \ref{sec:SOLVAR}, \algname~ is not restricted to a particular formulation of the projection operator $P_i$, which allows the user to select the projection implementation for a potential tradeoff of runtime and accruacy. In Figure~\ref{fig:algRunTime}, we show the optimization runtime for different projection implementations, and the two objectives (least squares and maximum likelihood).
We observe that the nearest and trilinear interpolations are considerablly faster than then using NUFFT for projections (where the FINUFFT library implementation  \cite{CUFINUFFTshih2021cufinufftloadbalancedgpulibrary} is used). Moreover, the difference in runtime between the two different objectives is practically unnoticable, despite the additional $O(r^3)$ term in the maximum liklihood objective. This is because the computional bottleneck still remains applying the projection (and back-projection $P_i^T$) operator.

\begin{figure}
	\centering
	\includegraphics[width=1\textwidth]{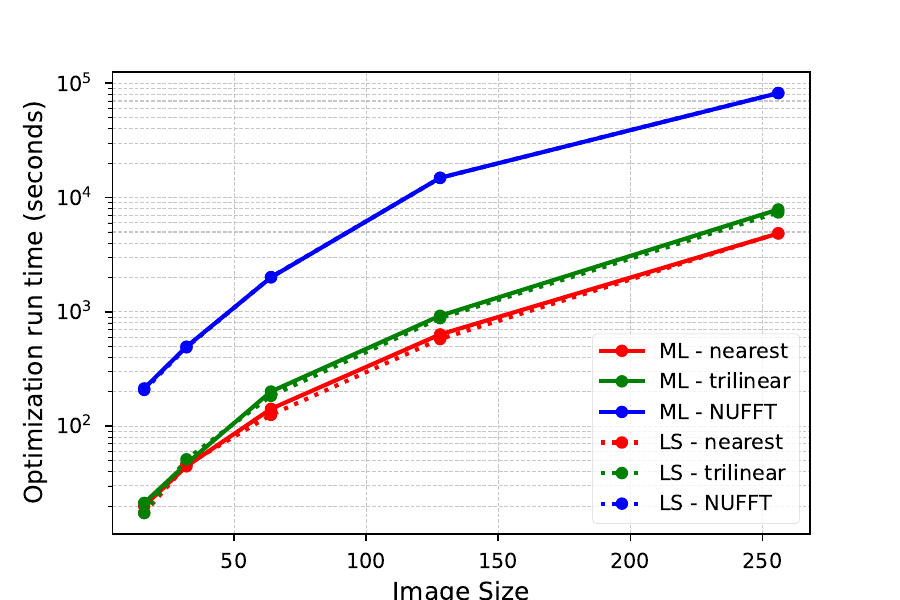}
	\caption{Run time (for estimating $r=10$ leading eigenvectors) as a function of image size using NVIDIA A100 GPU for a dataset of 100,000 images and using 40 epochs. ML and LS are the maximum likelihood and least squares estimators, respectively. Nearest and trilinear are the 0-th and first order interpolation of the projection operator in the Fourier domain (with up-sampling factor of 2);
    for a detailed description of Non-Uniform FFT (NUFFT) see \cite{nufft_doi:10.1137/S003614450343200X}.}
	\label{fig:algRunTime}
\end{figure}

\end{appendices}

\bibliographystyle{plain}
\bibliography{references}

\end{document}